\documentclass[sigconf]{aamas}  % do not change this line!

\usepackage{booktabs}

\usepackage{flushend}
\usepackage{multirow}
\setcopyright{ifaamas}  % do not change this line!
\acmDOI{}  % do not change this line!
\acmISBN{}  % do not change this line!
\acmConference[]{}{}{}  % do not change this line!
\acmYear{}  % do not change this line!
\copyrightyear{}  % do not change this line!
\acmPrice{}  % do not change this line!

\usepackage{graphics} % for pdf, bitmapped graphics files
\usepackage{epsfig} % for postscript graphics files
\usepackage{algorithm}
\usepackage{algpseudocode}
\usepackage{booktabs}
\usepackage[font=footnotesize,labelformat=simple]{subcaption}
\usepackage[font=footnotesize]{caption}
\newcommand{\figref}[1]{Fig.\ref{#1}}

\newcommand{\field}[1]{\mathbb{#1}}

\usepackage{hyperref}
\hypersetup{bookmarksopen,bookmarksnumbered,
pdfpagemode=UseOutlines,
colorlinks=true,
linkcolor=blue,
anchorcolor=blue,
citecolor=blue,
filecolor=blue,
menucolor=blue,
urlcolor=blue
}

\DeclareMathOperator*{\argmax}{arg\,max}
\DeclareMathOperator*{\argmin}{arg\,min}

\newcommand{\E}{\field{E}}

\newcommand{\inner}[1]{ \left\langle {#1} \right\rangle }
\newcommand{\one}{\boldsymbol{1}}
\newcommand{\order}{O}

\newcommand{\xxnote}[3]{}
\ifx\hidenotes\undefined
  \usepackage{color}
  \renewcommand{\xxnote}[3]{\color{#2}{#1: #3}\color{black}}
\fi

\title{Fair Contextual Multi-Armed Bandits: Theory and Experiments}

\author{Yifang Chen$^{*}$}
\affiliation{University of Southern California}
\email{yifang@usc.edu}

\author{Alex Cuellar}
\affiliation{Massachusetts Institute of Technology}
\email{alexcuel@mit.edu}

\author{Haipeng Luo}
\affiliation{University of Southern California}
\email{haipengl@usc.edu}

\author{Jignesh Modi}
\affiliation{University of Southern California}
\email{jigneshm@usc.edu}

\author{Heramb Nemlekar}
\affiliation{University of Southern California}
\email{nemlekar@usc.edu}

\author{Stefanos Nikolaidis}
\affiliation{University of Southern California}
\email{nikolaid@usc.edu}

\thanks{*all authors contributed equally}

\begin{abstract}
When an AI system interacts with multiple users, it frequently needs to make allocation decisions. For instance, a virtual agent decides whom to pay attention to in a group setting, or a factory robot selects a worker to deliver a part.
Demonstrating  \emph{fairness} in decision making is essential for such systems to be broadly accepted. We introduce a Multi-Armed Bandit algorithm with fairness constraints, where fairness is defined as a minimum rate that a task or a resource is assigned to a user. The proposed algorithm uses \emph{contextual} information about the users and the task and makes no assumptions on how the losses capturing the performance of different users are generated. We provide theoretical guarantees of performance and empirical results from simulation and an online user study. The results highlight the benefit of accounting for contexts in fair decision making, especially when users perform better at some contexts and worse at others. 
\end{abstract}

\begin{document}

\maketitle
\thispagestyle{empty}
\pagestyle{empty}

%%%%%%%%%%%%%%%%%%%%%%%%%%%%%%%%%%%%%%%%%%%%%%%%%%%%%%%%%%%%%%%%%%%%%%%%%%%%%%%%

%%%%%%%%%%%%%%%%%%%%%%%%%%%%%%%%%%%%%%%%%%%%%%%%%%%%%%%%%%%%%%%%%%%%%%%%%%%%%%%%

%-------------------------------------------------------------------------------
\section{Introduction}\label{sec:intro}
%-------------------------------------------------------------------------------
%As a general comment, I believe the algorithm names for Fair CB and UCB could be misleading.  I don't believe we ever use the UCB algorithm in this paper.  I think it would be better to have names that are derived from FTRL (Follow The Regularized Leader) since that is the algorithm we use for all of our experiments and studies.  Maybe we make a distinction between FTRL and Fair FTRL for the unfair and fair versions of the algorithm.  And potentially make a distinction between FTRL and context FTRL or cFTRL.  Under this naming scheme, for the user study we would be comparing Fair FTRL with Fair cFTRL.  I'm not saying these are the names we have to go with, but I think the current names may cause confusion and we should discuss them.  

We focus on the problem of an AI system assigning tasks or distributing resources to multiple humans, one at a time, while maximizing a given performance metric. %We focus on the problem of a robot iteratively distributing resources to multiple human teammates.  
For instance, a virtual agent decides whom to pay attention to in a group setting, or a factory robot selects a worker to deliver a part. 

%assists human workers by delivering parts. The robot wants to maximize overall performance by selecting the worker that would receive each part, while accounting  for worker's  experience,  skill  and fatigue. 

If there is clearly a user who outperforms everyone else, the solution to this optimization problem would result in the agent constantly selecting that user. This approach, however, fails to account that this may be perceived as unfair by others, which in turn may affect their acceptance of the system.
%However, this approach fails to consider that other workers may perceive such a distribution of resources as unfair, and consequently affect their interaction and trust in the system.  

How can we integrate \emph{fairness} in the agent's decisions? The aim of our work is to address this question. %How should the robot make decisions to take into account such issues of trus? Our work aims to answer this question.  
%~\cite{claure2019reinforcement}
Recent works~\cite{li2019combinatorial,claure2019reinforcement,patil2019achieving} have proposed multi-armed bandit algorithms for \textit{fair} task allocation, where fairness is defined as a constraint on the minimum rate of arm selection. A user study on an online Tetris game, where the computer (player) selects users (arms) based on their score, has shown that users' trust is significantly improved when a fairness constraint is satisfied~\cite{claure2019reinforcement}.

These works, however, have assumed that the performance of each user, observed in the form of a loss vector by the agent, follows a fixed distribution that is specific to that particular user. It thus fails to account that people may have different task-related skills. For instance, when making a pin, one worker may be specialized in cutting the wire, while another worker in measuring it. It also fails to account for cases where we can not make statistical assumptions about the generation of losses, for instance in an adversarial domain.
%or performance may change over time, for instance due to habituation or fatigue.  

%This work, however, has focused on \textit{stochastic} multi-armed bandits, relying on the assumption that the performance of each human worker, observed in the form of a loss vector by the robot, follows a \textit{fixed} distribution specific to this worker. It thus fails to account for cases where workers have different skills. For instance, when making a pin, one worker may be specialized in cutting the wire, while another worker in measuring it~\cite{}.   %For instance, when making a pin, one worker may be specialized in cutting the wire, while another worker specializes in measuring it~\cite{}.   

We generalize this work by proposing a fair multi-armed bandit algorithm that accounts for different \textit{contexts} in task allocation. The algorithm also does not make any assumption on how the loss vector is generated, allowing for applications in non-stationary and even adversarial settings. 

We provide theoretical guarantees on performance, as well as empirical results from simulations and a proof-of-concept online user study, where an algorithm assigns knowledge-based questions to participants from different cultural backgrounds. The results show the benefit of the proposed algorithm when allocating tasks fairly to different users, especially when they are better in some contexts and worse in others.

\section{Problem Definition} \label{problem}

We study the online learning problem of contextual bandits (CB) with fairness constraints. 
We assume $M$ possible contexts and $K$ available actions (arms), and use the notation $[M]$ and $[K]$ to denote the set $\{1, \ldots, M\}$ and $\{1, \ldots, K\}$.
For each time step $t=1,...,T$:
\begin{enumerate}
\item The environment first decides the context $j_t \in [M]$ and the loss vector $l_t \in [0,1]^K$.
\item The learner observes the context $j_t \in [M]$ and selects the action $i_t \in [K]$.
\item The learner suffers the loss $l_t(i_t)$. 
\end{enumerate}
We assume that the contexts $j_1, \ldots, j_T$ are i.i.d. samples of a fixed distribution $q \in \Delta_M$ which is known to the learner (see Section~\ref{sec:unknown} for extension to the case when $q$ is unknown).
However, we make no assumption on how the loss vectors $l_1, \ldots, l_T$ are generated, and in general $l_t$ could depend on the entire history before round $t$, which is a key difference compared to previous work~\cite{claure2019reinforcement}.

Let $\Delta_K$ be the set of distributions over $K$ arms.
Given the history up to the beginning of round $t$ and that context $j_t$ is $j$, we let $p_t^j \in \Delta_K$ be the conditional distributions of the player's selected arm $i_t$, for $j = 1, \dots,M$. We require the following \textit{fairness} constraint parameterized by $v \in (0,1/K)$:
\begin{equation}
\sum_{j=1}^{M}q(j)p_t^j(i)\geq v, \;\forall t,i,
\label{eq:fairness}
\end{equation}
that is, the marginal probability of each arm being pulled is at least $v$ for each time.

For notational convenience, we denote a collection of $M$ distributions over arms by $P=(p^1,...,p^M)$ and the feasible set of these collections in terms of the above constraint by: 

\begin{equation}\label{eqn:feasible_set}
\Omega = \begin{Bmatrix}
\left.\begin{matrix}
P = (p^1,...,p^M) 
\end{matrix}\left|\begin{matrix}
p^1,...,p^M \in \Delta_K \ \textrm{and}\\ 
\sum_{j=1}^M q(j)p^j(i) \geq v, \forall i \in [K]\end{matrix}\right.\right.
\end{Bmatrix},
\end{equation}
which is clearly a convex set and is non-empty since the uniform distribution (for all contexts) is always in the set.

The learner's goal is to minimize her regret, defined as the difference between her total loss and the loss of the best fixed distribution satisfying the fairness constraint: 
\begin{equation*}
\text{Reg} = \underset{P_*\in \Omega}{\textup{max}}\  \mathbb{E} \left [ \sum_{t=1}^{T}\inner{p_t^{j_t}-p_*^{j_t},l_t} \right].
\end{equation*}
Achieving sublinear regret $\text{Reg}=o(T)$ thus implies that in the long run the average performance of the learner is arbitrarily close to the best fixed distribution in hindsight.

\section{Background}\label{sec:background}
%!TEX root=main.tex

%-------------------------------------------------------------------------------
% \subsection{Contextual Multi-Armed Bandits}
%We study the online learning problem of contextual bandits (CB) defined by $M$ possible contexts and $K$ arms. At each time step $t = 1, \cdots, T$, the environment decides the context $c_{t} \in [M]$ and its corresponding loss vector $l_{t}$. The agent then has to choose an action $a_{t} \in [K]$ for the observed context $c_{t}$ with the goal of minimizing its regret. The regret is defined as the difference between the total loss incurred by the agent and the total loss for the best strategy $a^{*}$ in hindsight. %for the single best arm $a^{*}$ in hindsight.  
%$$ Reg = \sum_{t=1}^{T}l_{t}(a_{t}) - \min_{a^{*}\in K}\sum_{t=1}^{T} l_{t}(a_{t}^{*})$$
%The general multi-armed bandit problem that makes assumptions 
\noindent\textbf{Adversarial Bandits.} 
In the case when $M=1$ and $v=0$ (that is, only one context and no fairness constraint),
our problem is exactly the adversarial version of the classic Multi-armed Bandits (MAB) problem, first proposed in~\citep{auer2002nonstochastic} and extensively studied since then.
It is well-known that the minimax optimal regret is of order $\order(\sqrt{TK})$. 
The most common algorithm with optimal regret is Exp3~\citep{auer2002nonstochastic}, which can be regarded as a special case of the Follow-the-Regularized-Leader (FTRL) algorithm  when we choose the regularizer to be the negative entropy (see for example \cite{abernethy2015fighting}). \\

%Specifically, at each time step, the agent selects ``the best choice thus far'' with the addition of a regularization function $R$ that is convex and differentiable:
%$$ a_{t} = \argmin_{a \in K} \sum_{i = 1}^{t-1} l_{i}(a) + \Psi_t(a)$$
%We extend the FTRL algorithm for the case of fairness constraints in a special case of contextual bandit setting, where the number of contexts is limited.

\noindent\textbf{Contextual Bandits (without fairness).} 
When there are multiple contexts but no fairness constraint, with our regret definition there is no connection between the contexts, and the optimal algorithm is to treat each context separately and to run an individual instance of a standard MAB algorithm (such as Exp3) for each context (see Section~4 of~\cite{bubeck2012regret}).

We assume finite number of contexts and are interested in the case when $M$ is small.
There is a different line of research where $M$ could potentially be infinite, in which case a different measure of regret is studied or additional assumptions are made.
For example, in~\cite{auer2002nonstochastic, langford2007epoch, agarwal2014taming, SyrgkanisLuKrSc16}, the learner is given a fixed set of mappings from contexts to actions, and regret is defined in terms of the difference between the learner's total loss and the loss of the best mapping from the given set.
Other works make assumptions on how the losses are connected with the context.
Among those,  the linear assumption is the most common one, resulting in the so-called contextual linear bandit problem (e.g.~\citep{li2010contextual, chu2011contextual, abbasi2011improved, wu2018learning}).
Another common assumption is imposing some Lipschitz conditions~\citep{kleinberg2008multi, bubeck2011lipschitz, slivkins2014contextual, cesa2017algorithmic}. \\

\noindent\textbf{Fair Bandits.} %The concept of fairness in multi-armed bandits has been defined in several ways \cite{joseph2016fair, joseph2016fairness, liu2017calibrated, li2019combinatorial}.
Joseph et al. \cite{joseph2016fair,joseph2016fairness} are among the first to study fairness for bandits and draw inspiration from the idea of fair treatment suggested by Dwork et al. \cite{dwork2012fairness} which states that ``similar individuals should be treated similarly.'' The definition of fairness there is quite different from ours, in that a worse arm should not be picked compared to a better arm, despite the uncertainty on payoffs. The authors provide  a provably fair algorithm
for the linear contextual bandit problem.  Liu et al. \cite{liu2017calibrated} build upon this work to achieve smooth fairness, which requires arms with similar distributions to be selected with similar probabilities. They further define calibrated fairness, where an arm is selected with a probability equal to the probability of its loss being the lowest. These definitions are quite different from our notion of fairness which is a constraint on the minimum rate at which each arm is selected. 

Most relevant to ours is the work by Claure et al.~\cite{claure2019reinforcement}, where fairness is defined as a minimum rate on the selection of each arm, satisfied strictily throughout the task. Similarly, Li et al.~\cite{li2019combinatorial} define fairness as the minimum rate satisfied in expectation at the end of the task. Very recent work by Patil et al.~\cite{patil2019achieving} 
further extends this definition by denoting an unfairness tolerance allowed in the system. The aformentioned works focus on a stochastic MAB setting, where the losses are independent and
identically distributed. Instead, we propose an algorithm for the contextual MAB setting and we showcase the benefit of accounting for contexts in an online user study, where the system estimates the performance of players of different backgrounds in knowledge-based questions.

%Most relevant to ours is the work by Li et al. \cite{li2019combinatorial}, Claure et al.~\cite{claure2019reinforcement} 
%on sleeping bandits with fairness constraints. The fairness requirement is a minimum ratio for selection of an individual arm that is satisfied in expectation at the end of the task.

%Previous work \cite{claure2019reinforcement} imposes a similar minimum rate on selection of each arm, but it is satisfied strictly and anytime throughout the task. This paper applies the same concept of fairness to a problem with multiple contexts each having their own loss distributions. %should we use the word "distributions"? Since the algorithm is adversarial, it may be better to way "loss function" or something along those lines.  

%-------------------------------------------------------------------------------
\section{Algorithm}\label{sec:algo}
%!TEX root = main.tex
As mentioned earlier, without the fairness constraint, there is no connection among the contexts and the optimal algorithm is just to run $M$ instances of any standard MAB algorithm separately for each possible context.
For example, classic FTRL algorithm would compute for each context $j\in[M]$:
\begin{equation}\label{eqn:FTRL}
p_t^j = \argmin_{p \in \Delta_K} \sum_{s: j_s = j}\inner{p, \hat{l}_s} + \frac{1}{\eta}\sum_{i=1}^K \psi(p(i))
\end{equation}
at the beginning of round $t$,
where $\psi : [0,1]\rightarrow \mathbb{R}$ is some regularizer, $\eta > 0$ is some learning rate, and $\hat{l}$ is the standard unbiased importance-weighted estimator with:
\[
\hat{l}_s(i) = \frac{l_s(i)}{p_s^{j_s}(i)}\one\{i_s=i\}, \;\forall i\in[K].
\]
Upon observing the actual context $j_t$ for round $t$, the algorithm then samples $i_t$ from $p_t^{j_t}$.
Standard results~\citep{bubeck2012regret} show that the $j$-th instance of FTRL suffers regret $\order(\sqrt{|\{t: j_t = j\}|K})$, and thus the total regret is $\sum_{j=1}^M \order(\sqrt{|\{t: j_t = j\}|K}) = \order(\sqrt{TMK})$ via the Cauchy-Schwarz inequality.

With the fairness constraint, however, we can no longer treat each context separately.
A natural idea is to optimize jointly over the feasible set $\Omega$ defined in Eq.~\eqref{eqn:feasible_set}, that is, to find $P_t = (p_t^1, \cdots, p_t^M)$ at round $t$ such that:
\[
P_{t} = \argmin_{P \in \Omega} \sum_{s=1}^{t-1} \inner{p^{j_s},\hat{l}_s} + \frac{1}{\eta}\sum_{j=1}^M\sum_{i=1}^K \psi(p^j(i)).
\]
It is clear that when $v = 0$ (that is, no fairness constraint), the feasible set $\Omega$ simply becomes $\Delta_K \times \cdots \times \Delta_K$ and the joint optimization above decomposes over $j$ so that the algorithm degenerates to that described in Eq.~\eqref{eqn:FTRL}.
When $v\neq 0$, the algorithm satisfies the fairness constraint automatically and can be seen as an instance of FTRL over a more complicated decision set $\Omega$.

We deploy the standard entropy regularizer $\psi(p) = p\ln p$, used in the classic Exp3 algorithm~\cite{auer2002nonstochastic} for MAB.
See Algorithm~\ref{algorithm: known context} for the complete pseudocode.
We remark that even though unlike Exp3, there is no closed form for computing $P_t$,
one can apply any standard convex optimization toolbox to find $P_t$ when implementing the algorithm.

\begin{algorithm}[t]
 \caption{Fair CB with Known Context Distribution}
\label{algorithm: known context}
\begin{algorithmic}[1]
\State\textbf{Input:} learning rate $\eta > 0$, fairness constraint parameter $v$ 
\State\textbf{Define:} $\Psi(P) =  \frac{1}{\eta}\sum_{j=1}^M\sum_{i=1}^K \psi(p^j(i))$ where $\psi(p) = p\ln p$
    \For{$t=1,\ldots,T $}
        \State Compute $P_{t} = \argmin_{P \in \Omega} \sum_{s=1}^{t-1} \inner{p^{j_s},\hat{l}_s} + \Psi(P)$
        \State Observe $j_t$ and play $i_t \sim p_t^{j_t}$
        \State Construct loss estimator $\hat{l}_t(i) = \frac{l_t(i)}{p_t^{j_t}(i)}\one\{i_t=i\}, \;\forall i\in[K]$
    \EndFor
\end{algorithmic}
\end{algorithm}

We prove the following regret guarantee of our algorithm, which is essentially the same as the aforementioned bound for $v=0$.

\begin{theorem}\label{thm:regret}
    With learning rate $\eta = \sqrt{\frac{M\ln K}{TK}}$, Algorithm~\ref{algorithm: known context} achieves
    \begin{align*}
        \text{Reg} = \order\left(\sqrt{TMK\ln K}\right).
    \end{align*}
\end{theorem}

\begin{proof}
The proof follows standard techniques (such as~\cite{abernethy2015fighting}) once we rewrite our algorithm as FTRL in the space of $\mathbb{R}^{MK}$.
First we extend the loss vector $\hat{l}_t \in \mathbb{R}^{K}$ to a vector $L_t \in\mathbb{R}^{MK}$ by padding zeros to irrelevant coordinates.
Formally, $L_t = \hat{l}_t(i_t)e_{(j_t-1)K+i_t}$ where $e_1, \ldots, e_{KM}$ are standard basis vectors in $\mathbb{R}^{MK}$.
Further let $G_t = -\sum_{s=1}^t L_s$ be the negative cumulative loss estimator up to time $t$.
%Following our notation for $P$, we let $g^j_i$ be the $((j-1)K+i)$-th coordinate of $G$ for any vector $G \in \mathbb{R}^{MK}$.
Define 
$
\Psi^*(G) = \max_{P\in \Omega} \inner{P, G} - \Psi(P),
$
which is the convex conjugate of the function $\Psi(P) + \one_\Omega(P)$ where $\one_\Omega(P)$ is 0 if $P \in \Omega$ and $\infty$ otherwise.
With these notations we then have
\begin{align*}
P_t &= \argmin_{P\in \Omega} \inner{P, \sum_{s=1}^{t-1} L_s} + \Psi(P) \\
&= \argmax_{P\in \Omega} \inner{P, G_{t-1}} - \Psi(P) = \nabla \Psi^*(G_{t-1}). 
\end{align*}
Next, note that the loss estimators are unbiased since $\E[\hat{l}_t(i)] = \E\left[p^{j_t}_t(i) \times \frac{l_t(i)}{p^{j_t}_t(i)}\right] = \E[l_t(i)]$ for all $i \in [K]$.
We can thus rewrite the regret as $\text{Reg}= \E\left[\inner{P_*, G_T} + \sum_{t=1}^T \inner{\nabla \Psi^*(G_{t-1}), L_t}\right]$ where $P_* = \argmax_{P \in \Omega}\E\left[\inner{P, G_T}\right]$.
Recalling the Bregman divergence associated with $\Psi^*$ defined as 
\[
D_{\Psi^*}(G, G') = \Psi^*(G) - \Psi^*(G') - \inner{\nabla \Psi^*(G'), G - G'}.
\]
we further rewrite the regret as 
\begin{align*}
\text{Reg}
&= \E\left[\inner{P_*, G_T} + \sum_{t=1}^T \left(\Psi^*(G_{t-1}) - \Psi^*(G_t) + D_{\Psi^*}(G_t, G_{t-1})\right) \right] \\
&= \E\left[\inner{P_*, G_T} + \Psi^*(G_{0}) - \Psi^*(G_T) + \sum_{t=1}^T D_{\Psi^*}(G_t, G_{t-1})  \right].
\end{align*}
The first three terms can be bounded as (note $G_0 = \boldsymbol{0}$)
\begin{align*}
&\E\left[\inner{P_*, G_T} - \min_{P\in\Omega} \Psi(P) - \inner{P_*, G_T} + \Psi(P^*)\right] \\
&\leq - \min_{P\in\Omega} \Psi(P) = 
\frac{1}{\eta}\max_{P\in\Omega}\sum_{j=1}^{M}\sum_{i=1}^K p^j(i)\ln\frac{1}{p^j(i)}
\leq \frac{M\ln K}{\eta}
\end{align*}
where the last step uses the fact that the entropy of a distribution over $K$ items is at most $\ln K$.
It remains to bound $\E\left[D_{\Psi^*}(G_t, G_{t-1})\right]$.
By Taylor's theorem, there exists $\tilde{G}_t$ on the segment connecting $G_{t-1}$ and $G_t$ such that $D_{\Psi^*}(G_t, G_{t-1}) = \frac{1}{2}L_t^\top\nabla^2 \Psi^*(\tilde{G}_t) L_t$.
Moreover, using properties of convex conjugates (see for example~\cite{abernethy2015fighting}) we have $\nabla^2 \Psi^*(\tilde{G}_t)  \preceq \nabla^{-2} \Psi(\nabla\Psi^*(\tilde{G}_t))$.
Realizing that for any $P\in \Omega$, $\nabla^{-2}\Psi(P)$ is a diagonal matrix with $\eta P$ on the diagonal, we further bound the Bregman divergence by
\[
D_{\Psi^*}(G_t, G_{t-1}) \leq \frac{\eta}{2} \nabla\Psi^*(\tilde{G}_t)_{(j_t-1)K+i_t}  \hat{l}_t^2(i_t).
\]
Note that $\tilde{G}_t$ is the same as $G_{t-1}$ for all coordinates except the $((j_t-1)K+i_t)$-th one, where the value could only be smaller (if not equal) by the non-negativity of losses.
By the convexity of $\Psi^*$ (and thus monotonicity of $\nabla\Psi^*$), we then have
\begin{align*}
D_{\Psi^*}(G_t, G_{t-1}) &\leq \frac{\eta}{2} \nabla\Psi^*({G}_{t-1})_{(j_t-1)K+i_t}  \hat{l}_t^2(i_t) \\
&= \frac{\eta}{2}p^{j_t}_t(i_t) \hat{l}_t^2(i_t)
= \frac{\eta l_t^2(i_t)}{2 p^{j_t}_t(i_t)} \leq \frac{\eta}{2 p^{j_t}_t(i_t)}.
\end{align*}
Taking expectation on both sides gives $\E\left[D_{\Psi^*}(G_t, G_{t-1})\right] \leq \frac{\eta K}{2}$.
Finally, combing everything above we arrive at
\[
\text{Reg} \leq \frac{M\ln K}{\eta} + \frac{\eta TK}{2},
\]
which is of order $\order(\sqrt{TMK\ln K})$ with the optimal choice of learning rate $\eta = \sqrt{\frac{M\ln K}{TK}}$, finishing the proof.
\end{proof}
%-------------------------------------------------------------------------------

\label{sec:algorithm}

\section{Experiments}\label{sec:experiments}
%!TEX root = main.tex

This section illustrates different behaviors of the Fair CB algorithm, highlighting the interplay between choice of loss distributions, fairness and context. 
 
For each experiment we define the empirical performance of the algorithm in each experiment trial as one minus the average loss. $$\text{Performance} =  1 - \frac{\sum_{t = 1}^{T} l_{t}(i_t)}{T}.$$ 
In all experiments we set the learning rate as: $\eta = \sqrt{M\ln{K}/TK}$, following the theoretical result of section~\ref{sec:algorithm}. We run the experiments for the simplest case of two arms ($i_1$ and $i_2$) and two contexts ($j_1$ and $j_2$), while our insights generalize for an arbitrary number of contexts and arms.

We are motivated by settings where a system assigns resources to human users (arms) based on whether they succeed in a task or they exhibit a desired behavior. In Sections~\ref{subsec:fairness} and~\ref{subsec:contexts} we thus focus on the case where the loss induced by an arm $i$ under context $j$ follows a Bernoulli distribution parametrized by $\mu_{i,j}\in[0,1]$, so that $l_t(i)$ is 1 with probability $\mu_{i,j}$ and 0 with probability $1-\mu_{i,j}$, when the context is $j$. 
To showcase the advantage of our adversarial algorithm, in Section~\ref{subsec:adversarial} we also consider time-varying Bernoulli distributions.
The fairness level $v$ specifies the minimum rate that an arm is selected as defined in Eq.~\eqref{eq:fairness}.

\subsection{How Fairness Affects Performance}
\label{subsec:fairness}

%We expect that the fairness will not affect performace if the fairness constraint is not enforced, for instance when there is an arm $i$ that is best in at least one context and the probability of sampled contexts in which the arm is the best choice is greater than the fairness constraint. Otherwise, the fairness constraint will force selection of suboptimal arms, resulting in a decrease in performance.

With the presence of contexts, having a fairness constraint does not always lead to worse performance.
For instance, if for each arm, the probability of seeing the contexts in which this arm is the best is larger than $v$, then the fairness constraint can be satisfied trivially by picking the best arm for each context and the performance is also the best.
However, in the case where the fairness constraint forces the algorithm to select suboptimal arms, larger value of $v$ unavoidably leads to worse performance.
Below we demonstrate this phenomenon empirically with our fair CB algorithm. \\

\begin{figure}[t]
\centering
\begin{subfigure}{0.48\linewidth}
\includegraphics[width=\linewidth]{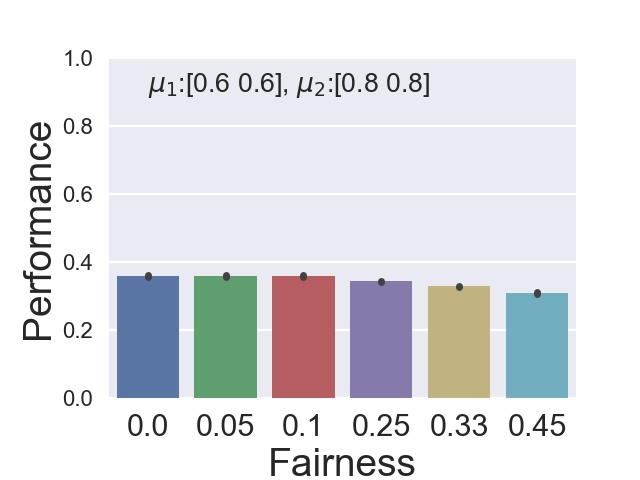}
\caption{Arm $i_{1}$ is better at both contexts} \label{fig:exp1_a}
\end{subfigure}
\begin{subfigure}{0.48\linewidth}
\includegraphics[width=\linewidth]{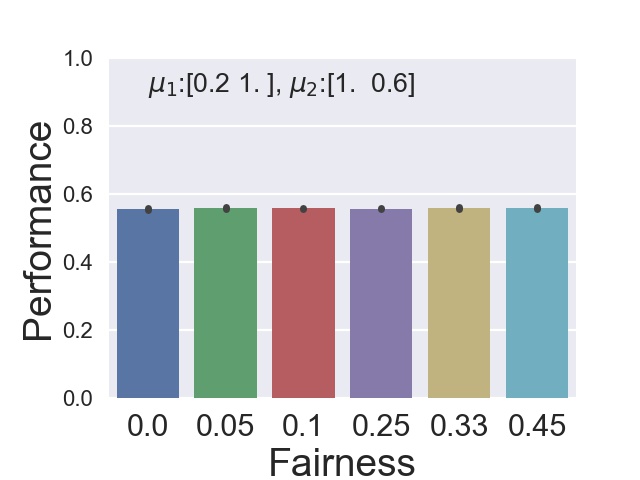}
\caption{Arm $i_{1}$ better only at context $j_{1}$} \label{fig:exp1_b}
\end{subfigure}
\vspace{-1 ex}
\caption{Performance of algorithm for different levels of fairness for $T=2000$, averaged over 100 simulations.}
\label{fig:exp1}
\vspace{-1 ex}
\end{figure}

\noindent\textbf{Even distribution of both contexts:}
We first let the contexts be distributed evenly, that is, $q(j_1) = q(j_2) = 0.5$.

%\noindent\textbf{H1.} \textit{Increasing fairness will result in worse performance when one arm is better than the other arm in both contexts.}
 If each arm is better than the other in one of the contexts, we expect that fairness does not affect the performance of an optimal algorithm, since the probability of the context occuring -- and thus that arm being selected -- is $q(j) = 0.5$ which is always greater than a fairness constraint $v \in \left(0,\frac{1}{K}\right)$. On the other hand, if one arm is better than the other in both contexts, we expect the algorithm to enforce the fairness constraint and choose the weakest arm with the minimum rate in at least one of the contexts.

\begin{figure}[t!]
\centering
\begin{subfigure}{0.47\linewidth}
\includegraphics[width=\linewidth]{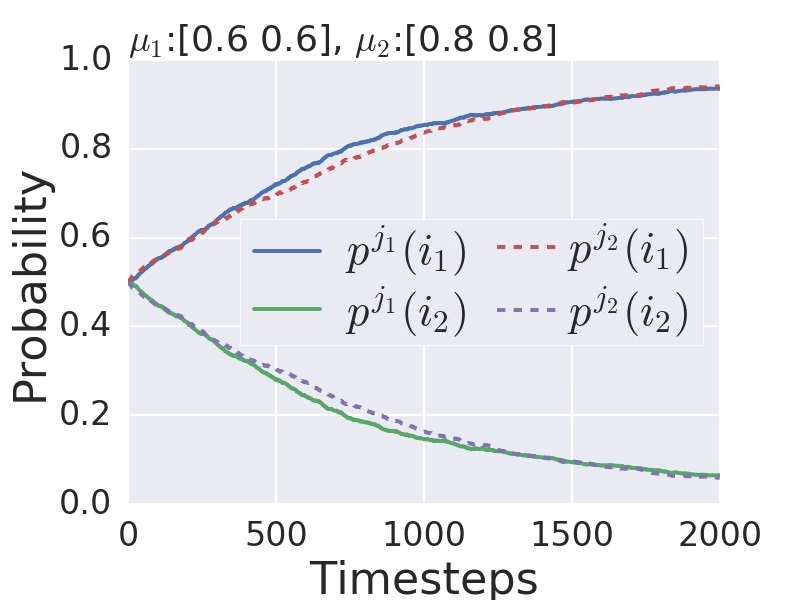}
\caption{$v = 0$} \label{fig:exp1_prob_a}
\end{subfigure}
\begin{subfigure}{0.47\linewidth}
\includegraphics[width=\linewidth]{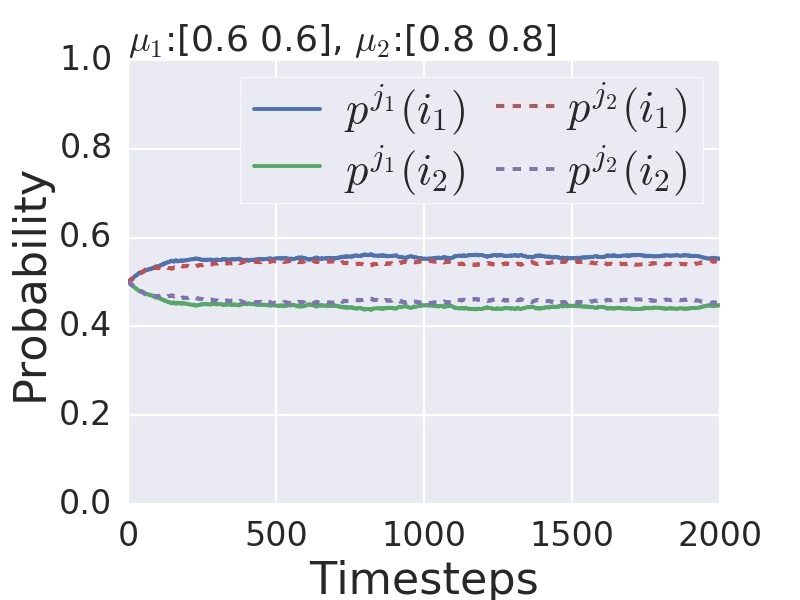}
\caption{$v = 0.45$} \label{fig:exp1_prob_b}
\end{subfigure}
\caption{Probabilities of pulling an arm over time averaged over 100 simulations when one arm is better in both contexts.}
\label{fig:exp1_arm_best_both}
\end{figure}

\begin{figure}[t!]
\begin{subfigure}{0.47\linewidth}
\includegraphics[width=\linewidth]{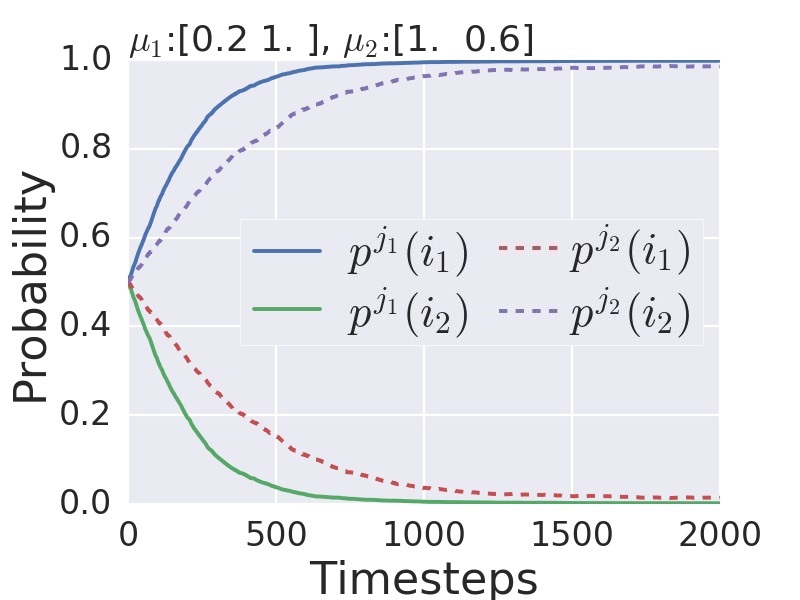}
\caption{$v = 0$} \label{fig:exp1_prob_c}
\end{subfigure}
\begin{subfigure}{0.47\linewidth}
\includegraphics[width=\linewidth]{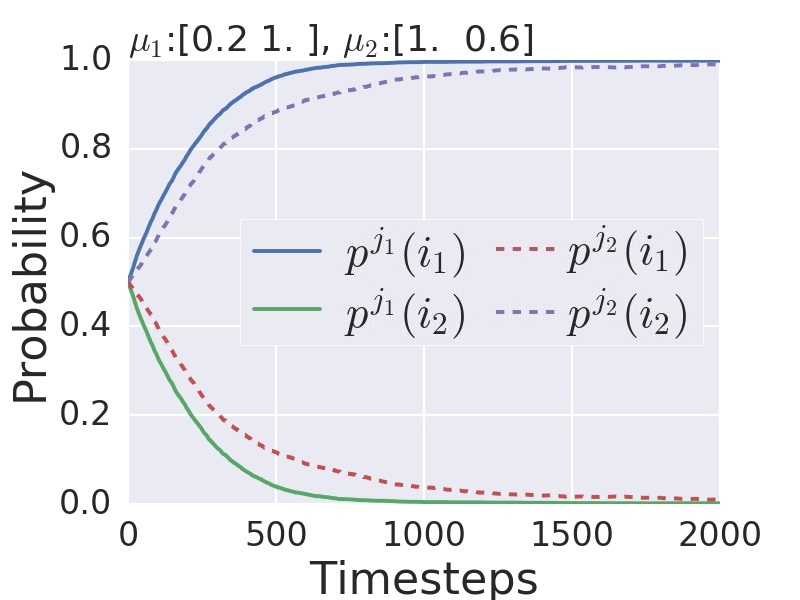}
\caption{$v = 0.45$} \label{fig:exp1_prob_d}
\end{subfigure}
\caption{Probabilities of pulling an arm over time averaged over 100 simulations when one arm is better in one context and worse in another.}
\label{fig:exp1_arm_best_one}
\vspace{-1 ex}
\end{figure}

%\paragraph*{Analysis} 
\emph{Arm 1 is better in both contexts}. For instance, we let $\mu_{1} = (\mu_{i_1,j_1}$, $\mu_{i_1,j_2}) = (0.6, 0.6)$  be the expected values of the loss distributions for contexts 1 and 2 for arm 1, and $\mu_{2} = (\mu_{i_2,j_1},\mu_{i_2,j_2}) = (0.8, 0.8)$ for arm 2. We run the algorithm in simulation for varying levels of fairness. We expect that increasing fairness results in selecting the suboptimal arm ($i_2$) with increasing frequency, which subsequently increases the total loss. 

Fig.~\ref{fig:exp1_a} shows the performance of our algorithm for six different values of $v$ over $T=2000$ rounds (and averaged over $100$ simulations).
As expected, the performance degrades as $v$ gets larger.
Note that performance was similar across the first three fairness levels. We attribute this to the inherent exploration of the FTRL algorithm from the regularization term and the relatively small difference between the expected losses $\mu_{1}$ and $\mu_{2}$ of the two players. 
A linear regression established that the fairness level significantly predicted performance, with $F(1, 598) = 1168.5, p < .0001$~\cite{kutner2005applied} and fairness accounted for $66.1\%$ of the explained variability in performance. The regression equation was: predicted performance $= 0.37 - 0.11 v$.

\figref{fig:exp1_arm_best_both} shows the assigned probabilities by the algorithm for every timestep, averaged over 100 simulations. Since $i_1$ is better than $i_2$ in both contexts, it eventually gets selected with probability close to 1 in both contexts when fairness $v=0$ and with probability 0.55 when fairness $v=0.45$.

%A one-way ANOVA performed on the performance of Fair CB in this example across all the fairness levels, indicates that the fairness has a significant effect on the performance ($p<0.01$). Thus when one player in better than the other in both contexts, increasing the fairness level leads to a decrease in the performance.
\emph{There is no arm that is better in both contexts}. In this case fairness level does not affect the performance of our algorithm, as shown in \figref{fig:exp1_b}, where $\mu_{1} = (0.2, 1.0), \mu_{2} = (1.0, 0.6)$. 

 \figref{fig:exp1_arm_best_one} shows the assigned probabilities over time. Regardless of the fairness parameter, since $i_1$ is better than $i_2$ in $j_1$ but worse in $j_2$, $i_1$ will be selected with probability close to 1 for $j_1$ and $i_2$ with probability close to 1 for $j_2$. Since $j_1$ and $j_2$ are distributed with probability 0.5, the fairness constraint is naturally satisfied. \\

\noindent\textbf{Uneven distribution of contexts:}
%\subsection{Uneven Context Distribution}
We then examine the general case where contexts are distributed with different probabilities. 
%Without loss of generality, we assume $K=M=2$ arms and contexts. 
%We make the following hypotheses: 
We expect that increasing fairness will result in worse performance when one arm is better than the other arm in both contexts, or when one arm $i_1$ is better than the other arm $i_2$ in only one context $j_1$ and $v > q(j_2)$. We let $q(j_1)=0.9,q(j_2) = 0.1$ be the distribution of the two contexts.

\begin{figure}[t!]
\centering
\includegraphics[width=0.49\linewidth]{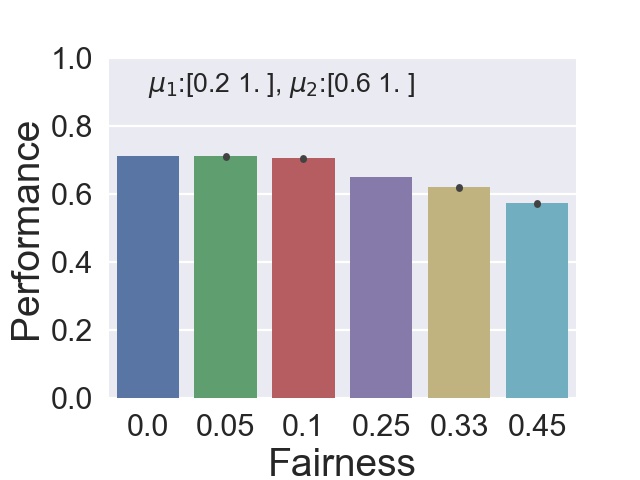}
\includegraphics[width=0.49\linewidth]{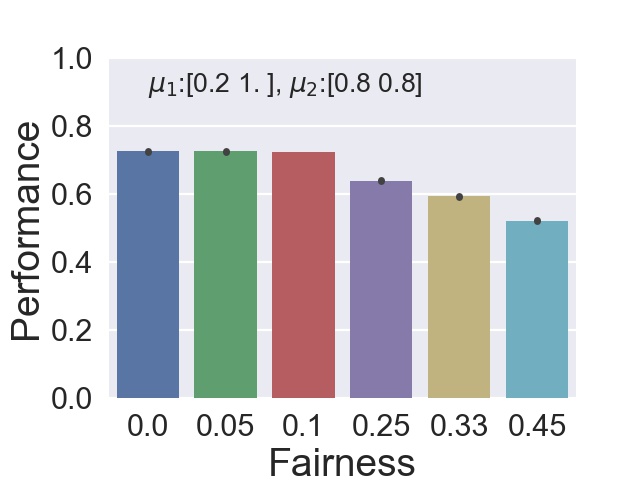}
\\
\includegraphics[width=0.49\linewidth]{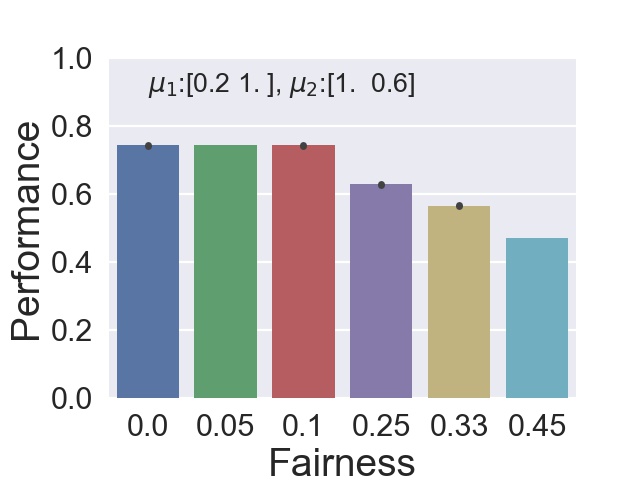}
\includegraphics[width=0.49\linewidth]{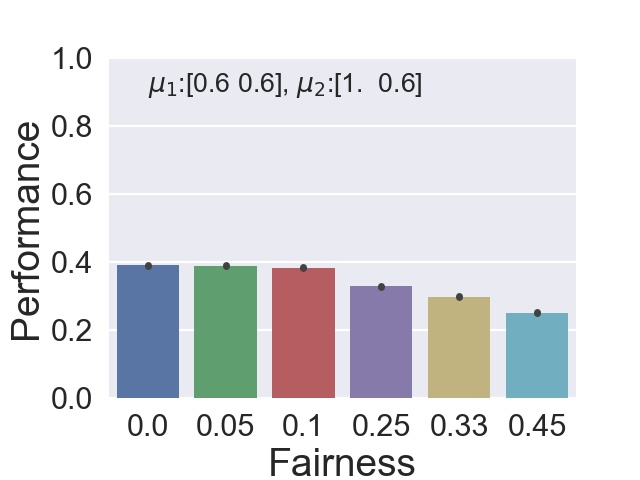}
%%\\
%\includegraphics[width=0.49\linewidth]{fig/sim/exp4_7.jpg}
%\includegraphics[width=0.49\linewidth]{fig/sim/exp4_8.jpg}
\caption{Performance for 2-arm 2-context with $q(j_1)=0.9,q(j_2)=0.1$, $T=10000$, averaged over 100 simulations.}
\label{fig:exp4_3}
\vspace{-2 ex}
\end{figure}

 The case for one arm being better in both contexts follows the same reasoning as before. On the other hand, if one arm $i_1$ is better than the other arm in one of the contexts $j_1$ with probability $q(j_1)$, we expect increasing fairness to reduce performance for $v > q(j_2) = 0.1$.

Indeed, for different combinations of $\mu_{i_{1},j_{1}},\mu_{i_{1},j_{2}}\in\{0.2,0.6,1\},$ $\mu_{i_{2},j_{1}},\mu_{i_{2},j_{2}}\in\{0.6,0.8,1\}$, a multiple regression model statistically significantly predicted performance, with $F(2, 2397) = 7294, p < .0001,~\textrm{adj.}\ R^2 = 0.86$ and fairness being a significant predictor ($p<0.001$). Fig.~\ref{fig:exp4_3} shows the performance for different configurations. We see that indeed fairness starts decreasing the performance once $v>0.1$.

 Overall, our analysis shows that, for any number of contexts and arms, fairness matters if the fairness constraint enforces an arm to be pulled in a context that is not optimal, which occurs either when there is no context where the arm is optimal, or when the probability of the context(s) that the arm is optimal is smaller than the probability imposed by the fairness constraint.

\subsection{The Importance of Contexts}
\label{subsec:contexts}

\begin{figure}[hbt]
\centering
\begin{subfigure}{0.47\linewidth}
\includegraphics[width=\linewidth]{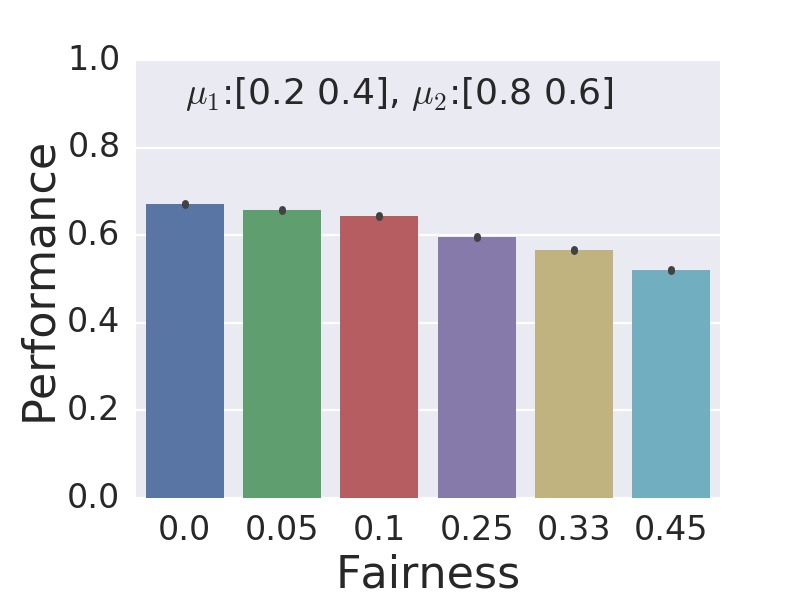}
\caption{Non-contextual FTRL} \label{fig:exp6_1_nocontext}
\end{subfigure}
\begin{subfigure}{0.47\linewidth}
\includegraphics[width=\linewidth]{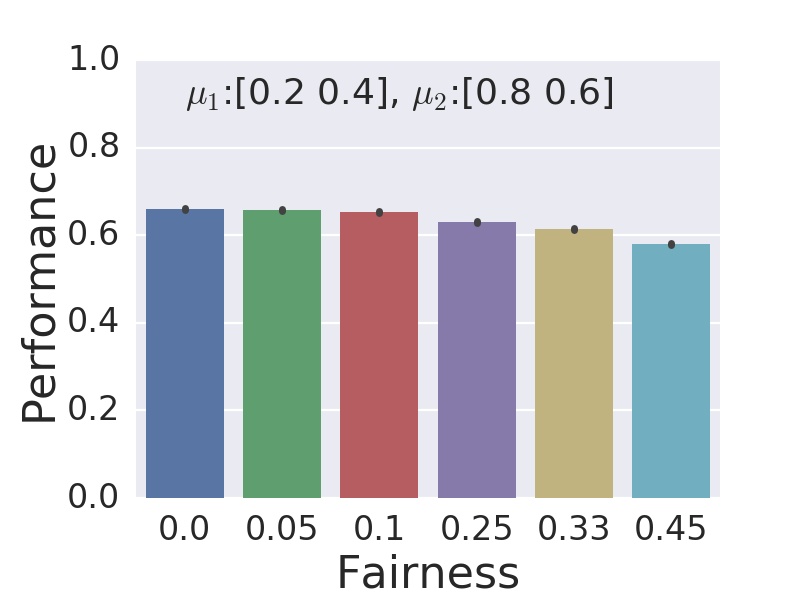}
\caption{Fair CB} \label{fig:exp6_1_context}
\end{subfigure}
\caption{Performance when $i_1$ is better in both contexts, $q(j_1) = q(j_2) = 0.5$ and $T=2000$, averaged over 100 simulations.}
\label{fig:exp6_1}
\vspace{-1 ex}
\end{figure}

\begin{figure}[hbt]
\centering
\begin{subfigure}{0.47\linewidth}
\includegraphics[width=\linewidth]{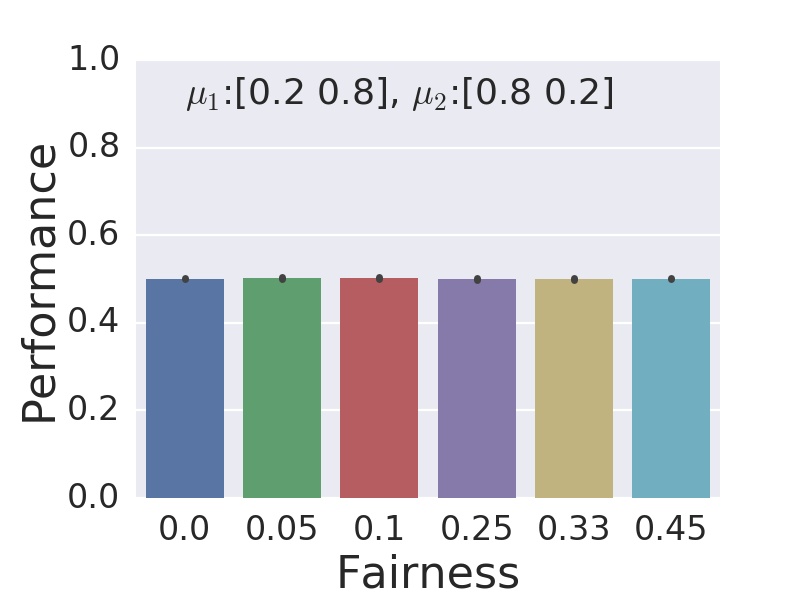}
\caption{Non-contextual FTRL} \label{fig:exp6_2_nocontext}
\end{subfigure}
\begin{subfigure}{0.47\linewidth}
\includegraphics[width=\linewidth]{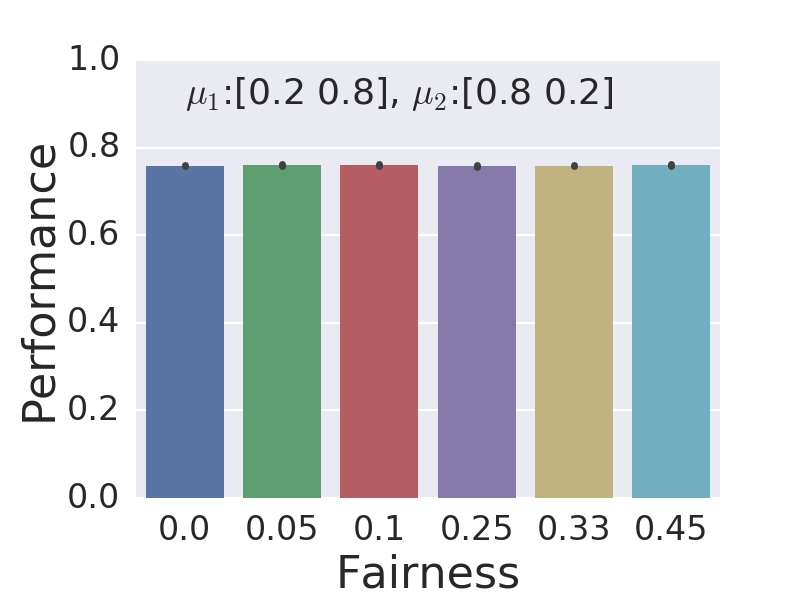}
\caption{Fair CB} \label{fig:exp6_2_context}
\end{subfigure}
\caption{Performance when $i_1$ is better in one of the contexts only, $q(j_1) = q(j_2) = 0.5$ and $T=2000$, averaged over 100 simulations.}
\label{fig:exp6_2}
\vspace{-1 ex}
\end{figure}

% \subsection{Comparison to Non-Contextual FTRL}
To illustrate the importance of contexts, we compare to an FTRL algorithm that ignores the context (equivalently, our algorithm with $M=1$). 
We consider even distribution among the two contexts.
First, we examine the case where one arm is better than the other in both contexts: $((\mu_{i_{1},j_1},\mu_{i_{1},j_2}) = (0.2,0.4),(\mu_{i_{2},j_1},\mu_{i_{2},j_2}) = (0.8,0.6)).$ Fig.~\ref{fig:exp6_1} shows the result for increasing values of fairness. While for 0 fairness there is no noticeable difference, as fairness increases, we observe that our Fair CB performs better. A one-way ANOVA for $v=0.45$ showed a significant effect of the choice of algorithm on performance ($F(1,198)=1197.43, p<0.0001$). Despite arm $i_{1}$ being better than arm $i_{2}$ in both contexts, we see a difference in performance, since the difference between the two arms' loss is much higher for the first context than the second. The contextual algorithm recognizes this disparity and selects to impose the fairness constraint in the second context rather than in both contexts. 

Fig.~\ref{fig:exp6_2} shows another result when one player is better in one context and worse in the other ($(\mu_{i_{1},j_1},\mu_{i_{1},j_2}) = (0.2,0.8),(\mu_{i_{2},j_1},\mu_{i_{2},j_2}) = (0.8,0.2)$). We observe that the contextual algorithm outperforms the baseline in all fairness levels, since it distributes the arms to different contexts while satisfying the fairness constraint. 
% while fairness does not affect performance as we expected from the previous sections.

\subsection{Adversarial Losses}
\label{subsec:adversarial}

An advantage of the Fair CB algorithm is that it makes no assumptions on how losses are generated. This contrasts previous work on fair task allocation~\cite{li2019combinatorial,claure2019reinforcement,patil2019achieving}, which assume a fixed distribution. 

 %We switch between $(\mu_{i_{1}}, \mu_{i_{2}}) = (0.1, 0.9)$ and $(\mu_{i_{1}}, \mu_{i_{2}}) = (0.9, 0.1)$ every time the learner incurs a loss of 0. 
To showcase this advantage, we compare our algorithm with Fair UCB, which assumes a stochastic setting and implements the standard UCB algorithm with a minimum pulling rate constraint (fairness) for each arm. While different implementations of Fair UCB were proposed independently by Claure et al.~\cite{claure2019reinforcement} and Patil et al.~\cite{patil2019achieving}, we use the former stochastic-rate constrained UCB implementation. Since we wish to focus on the effect of adversarial losses on performance, we used only one context ($M=1$) in both algorithms.
 
To simulate an adversarial setting, we generate the loss vector as follows: every time the learner incurs a loss of 0, the loss distribution switches between $(\mu_{i_{1}}, \mu_{i_{2}}) = (0.1, 0.9)$ and $(\mu_{i_{1}}, \mu_{i_{2}}) = (0.9, 0.1)$ (note that the index for $j$ is omitted here since $M=1$).

%For instance, we switch between the losses $(\mu_{i_{1}}, \mu_{i_{2}}) = (0.1, 0.9)$ and $(\mu_{i_{1}}, \mu_{i_{2}}) = (0.9, 0.1)$ every time the learner incurs a loss$=0$. 

We evaluate the performance of our algorithm and Fair UCB for different levels of fairness. A two-way ANOVA comparing the main effects of algorithm selection (Fair UCB and Fair CB) and fairness level ($v$) on performance shows a significant difference for both algorithms ( $F(1, 1188) = 1926.4$, $p<0.001$, Fair UCB M = $0.45$, SE = $0.0017$,  Fair CB M = $0.496$, SE = $5.46e-4$) and fairness ($F(5, 1188) = 220.41$, $p<0.001$). There was a significant interaction between the effects of algorithm selection and fairness ($F(5, 1188) = 211.46$, $p<0.001$).

Fig.~\ref{fig:adversarial} shows the performance of the two algorithms. We observe that fairness does not affect performance for Fair CB, since the switching loss vector makes the algorithm already quite conservative in the arm selection. On the contrary, Fair UCB has poor performance when fairness is small, while performance improves for increasing levels of fairness. This is because large fairness level makes the algorithm rely less on the UCB bound which is exploited by the adversary in this setting.

\begin{figure}[hbt]
\centering
\includegraphics[width=0.48\columnwidth]{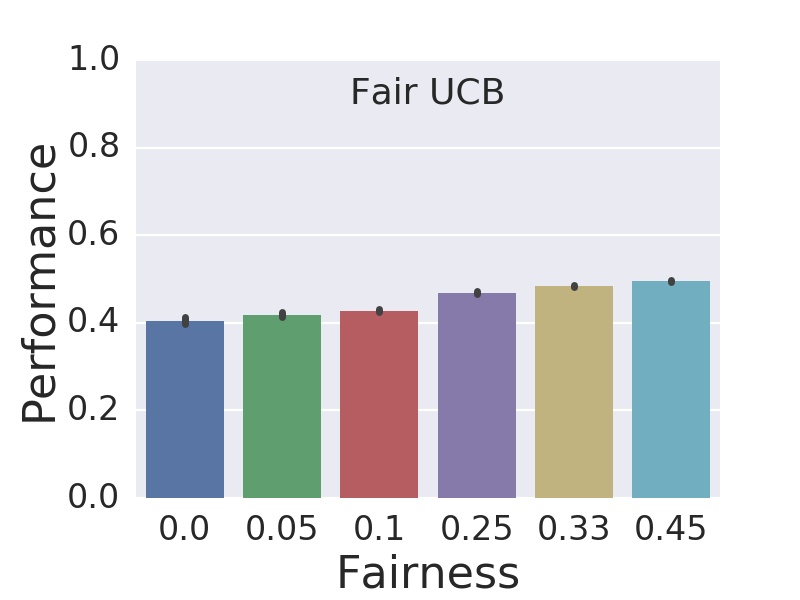}
\includegraphics[width=0.48\columnwidth]{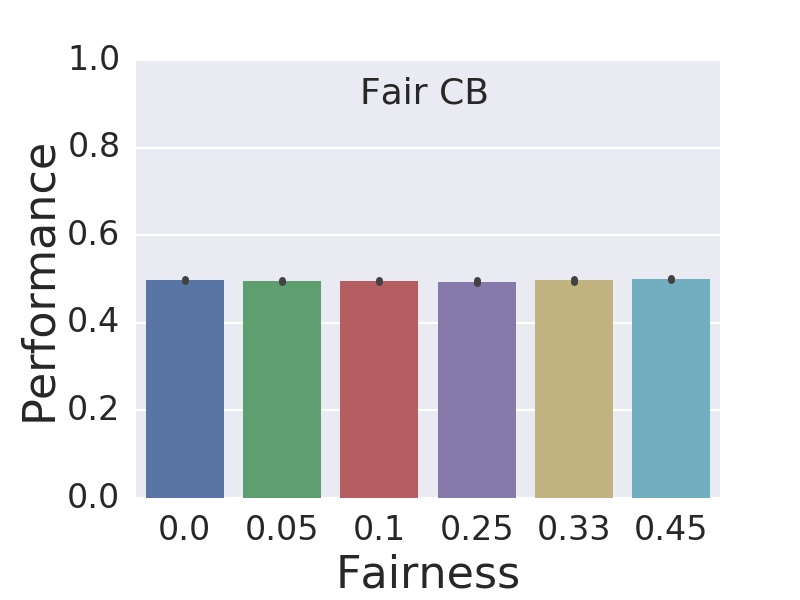}
\caption{Performance of Fair UCB and Fair CB algorithms for 2-arm 1-context problem with adversarial losses, with $T=1500$ averaged over 100 simulations.}
\label{fig:adversarial}
\vspace{-1 ex}
\end{figure}

% %-------------------------------------------------------------------------------
% \subsection{Fairness Affects Performance}
% \input{experiments-fairness.tex}
% %
% %-------------------------------------------------------------------------------
% \section{The importance of Contexts}
% \input{experiments-contexts.tex}

% \section{Adversarial Losses}
% \input{non-stationary.tex}

\section{User Study}
%!TEX root = main.tex

%-------------------------------------------------------------------------------
We wish to assess whether accounting for contexts when distributing resources fairly results in a better performance. Results from section~\ref{subsec:contexts} show that Fair CB is particularly beneficial when the arms are better in one context and worse in another. Therefore, we design a proof-of-concept online user study, where we expect participants to perform better in different contexts. 

In our study, the system has to assign knowledge-based questions from different topics to two users, one at a time, so that the number of correct answers is maximized. We compare the fairness and performance of Fair CB, with the non-contextual FTRL algorithm (i.e. Baseline) that does not consider context while assigning questions. %This study helps us to examine how the algorithm assigns questions of a particular topic to real users that are knowledgeable in that topic while maintaining fairness. 

\subsection{Experimental Setup}

\textit{Methodology:} We created an online quiz where users have to identify states and famous people from either USA or India, which are the 2 contexts. We paired two users to simultaneously take the quiz by matching users indicating India as their country of origin with users indicating the United States. We did this with the expectation that users from India would be better in questions related to their country than users from USA and vice versa.

We had two quizzes, each assigned to one of the algorithms (non-contextual FTRL or Fair CB). Each quiz had a fixed set of 44 questions evenly distributed between the two topics (20 questions about India, 20 about USA in alternating order). The first four questions of each quiz were equally divided among the two players for initialization. For each question, users had 10 seconds to select one out of four candidate answers.

We adopted a within-subjects design, where the same pair of users took both quizzes, one running the Fair CB algorithm and other running the Baseline algorithm.
We counter-balanced the assignment of quizzes to algorithms. While we did not expect any learning effects, since the quizzes included knowledge-based questions, we had a training section where subjects answered example questions and we also counter-balanced the order of the two algorithms.
%The order of the algorithms and quizzes was alternated to nullify any ordering effects.

% \begin{figure}[hbt]
% \centering
% \includegraphics[width=0.8\columnwidth]{fig/quiz_usa.png}\\
% \includegraphics[width=0.8\columnwidth]{fig/quiz_ind.png}
% \caption{Sample questions from the quiz based on knowledge of states from USA (top) and India (bottom).}
% \label{fig:quiz}
% \vspace{-1 ex}
% \end{figure}

\textit{Algorithm:} In this experiment we had two contexts $M = \{1,2\}$ and two human participants $K = \{1,2\}$. We set the fairness parameter $v$ to $0.33$. We tuned the learning rate for both algorithms to $\eta=0.25$.

To reduce variance from sampling, we implemented the Fair CB algorithm with deterministic schedules by setting a ``window'' of 10 questions, 5 for each context in alternating order, and we assigned participants to questions deterministically, based on the output of each algorithm. For instance, if $p^{j_1}(i_1) = 0.6$ for context 1 and $p^{j_2}(i_1) = 1.0$ for context 2, we assigned 3 of the 5 questions of context 1 to participant 1, all 5 questions of context 2 to participant 1, and the remaining questions to participant 2.

 At the end of that window the system received the loss values for each question corresponding to the context and participant, and updated the participant probabilities. Since we had a total of 44 questions, the algorithm performed 4 updates. 
 %After each update, the system assigned the turns for the participants for the next 10 questions. 
%The quiz is run over a time horizon of $T = \{1,2,3,4\}$. At each time step $t \in T$, the system assigns the the next 10 questions between the players based on the current player probabilities. 

\textit{Hypotheses:} We make the following hypothesis:

\noindent\textbf{H1.} \textit{Fair CB algorithm will perform better than the Baseline algorithm}. Since we expect users to be more knowledgeable in one of the contexts and less knowledgeable in the other context, we expected that Fair CB would result in better performance, compared to an algorithm that assesses users based on their  performance in both contexts together. We base this on the results from the simulations in section~\ref{subsec:contexts}.
%As the users can be knowledgeable in different topics, choosing users based on the topic should result in a better performance than simply choosing the better user irrespective of their background. 

\noindent\textbf{H2.} \textit{Participants' subjective responses will not be worse in the Fair  CB algorithm, compared to the baseline}. Since both algorithms account for fairness, we expected users' responses for the Fair CB to be at least as good as in the baseline case. 

We note that we did not compare against different fairness levels, since simulations  in section \ref{subsec:fairness} show that fairness matters only when one arm is better at both contexts, which we expect to happen infrequently in this study. We refer the reader to previous studies~\cite{claure2019reinforcement} which highlight the effects of fairness on users' perceived fairness and trust in the system.
%\noindent\textbf{H6.} \textit{Perceived fairness of the Fair CB algorithm will not be worse than that of the Baseline algorithm.} Previous work~\cite{claure2019reinforcement} has shown that accounting for fairness has a positive effect on users' trust in the system.

%The constraint on the minimum rate at which a player is chosen should result in a perceived fairness that is not inferior to that of the Baseline.

\begin{table*}[hbt]
\centering
\begin{tabular}{@{}lll@{}}
\toprule
Factor & Question No. & Question \\ \midrule
\multirow{2}{*}{Fairness} & Q1. & How FAIR or UNFAIR was it for YOU that the computer gave you the designated number of questions? \\
 & Q2. & How FAIR or UNFAIR was it for your PARTNER that the computer gave them the designated number of questions? \\
Trust & Q3. & How much do you trust the computer to make a good decision about the distribution of questions? \\ \bottomrule
\end{tabular}
\caption{Survey questions answered for both algorithms after the quiz.}
\label{tab:survey}
\end{table*}

\begin{table*}[t]
\resizebox{\textwidth}{!}{%
\begin{tabular}{ll}
\hline
Example Quote \\ \hline
%Noticed no difference & ``No.", ``I didn't really notice a difference." \\
%\multirow{4}{*}{Noticed a difference} \\ 
 ``Maximum from US based question to me while other India based" \\
  ``I thought I got more if I was right"\\ 
  ``The other person got way more questions than me" \\
  ``I feel my partner had more in section B" (Sec. B - Baseline) \\
  ``There seemed to be fewer questions in a row for each of us in Set B." (Set B - Fair CB) \\
  ``Part 1 seemed to do a much better job of giving questions about the US to me, and  questions about India to my partner." (Part 1 - Fair CB) \\
 ``The first was more even, in the 2nd the other player got a lot more questions" (1st - Fair CB) \\
 % \begin{tabular}[c]{@{}l@{}}``Yes towards the end of the second set the questions seemed to switch to the other player if you got one right instead of switching after \\ getting one wrong." (2nd set - Fair CB)\end{tabular} \\ 
  \hline
\end{tabular}%
}
\caption{Participants response to the question ``Did you notice a difference in how each part distributed questions?''}
\label{tab:comments}
\end{table*}

\textit{Measures:} We recorded the participants' performance, the number of the questions assigned, the loss values corresponding to the participant responses, and the probabilities estimated at each time step. We additionally asked participants questions related to their perceived fairness and trust in the system, using survey questions (Table~\ref{tab:survey}), where each response was measured on a seven-point Likert scale.

\textit{Procedures:} We recruited participants using Amazon Mechanical Turk (AMT) and used Qualtrics to create and record the survey responses. The AMT participants were instructed that they would be paired with another person to take the quiz together and the computer would decide who gets to answer a particular question. After the quiz the participants were redirected to the survey, where they answered questions about their experience. The study was approved by the Institutional Review Board of our University.
%The first four questions of the quiz would be equally divided, and therefore we initialize a loss value for each player in both contexts.  After the first four questions, the algorithm would choose player distribution of questions for the next 10 turns.  

\textit{Participants:} We recruited 80 participants (40 pairs) from AMT. We removed data from 3 pairs because they did not complete the quiz. The final dataset had $N=74$ participants (37 - US, 37 - India). 

\subsection{Results}

\subsubsection{Performance}
\begin{figure}[hbt]
\centering
\includegraphics[width=0.48\columnwidth]{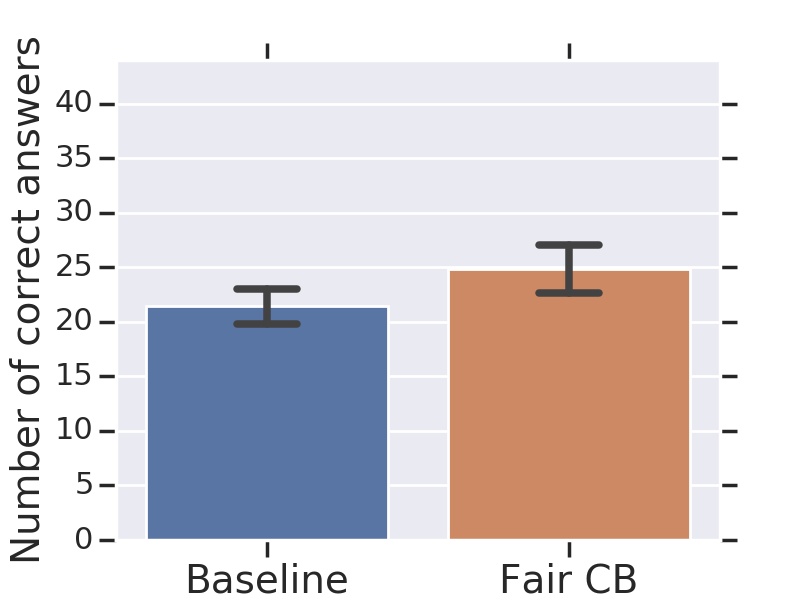}
\includegraphics[width=0.48\columnwidth]{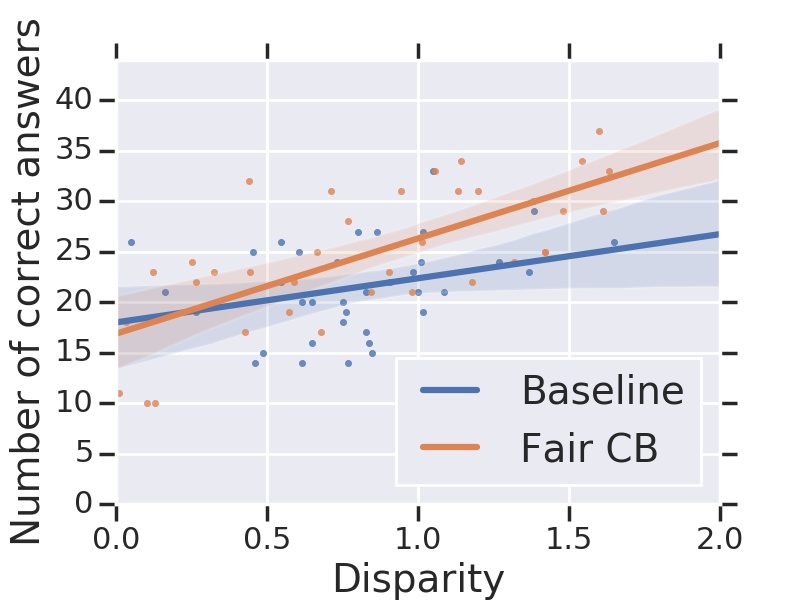}
\caption{Performance of Fair CB algorithm compared to Baseline.}
\label{fig:disparity}
\vspace{-1 ex}
\end{figure}

We measure the performance of each algorithm by the total number of questions answered correctly for each quiz. A paired t-test showed a statistical difference ($t(36) = -3.308, p = 0.002$) in the performance of the users for Baseline (M = 21.472, SE = 0.777) and Fair CB (M = 24.833, SE = 1.137) conditions. On average, the users answered $48.8\%$ questions correctly in the Baseline and $56.44\%$ questions correctly in the Fair CB conditions. We found no significant effect of the set of questions on performance. This result supports hypothesis H1. 
%We also compared the performance of users in the first (Set A) and second (Set B) set of questions, irrespective of the algorithm, to evaluate if there was any difference in the difficulty of the question sets. A paired t-test showed no statistical difference ($t(36)=-0.843, p=0.404$) is the users' performance in Set A (M = 22.66, SE = 1.06) and Set B (M = 23.63, SE = 0.956). 

A post-hoc analysis of the data shows that the difference in performance was larger  when one participant was much better than the other in one of the contexts. 
%The Fair CB algorithm assigns more questions in a context to the player performing better in that context. 
We show this by defining the \textit{disparity} between the participants as the average difference in participant performances for each context. Higher disparity means that one participant was much better in one of the contexts and worse in the other context: %\snnote{needs proper linear regression}
$$\delta= \left| (\hat{\mu}_{i_1,j_1}-\hat{\mu}_{i_2,j_1}) - (\hat{\mu}_{i_1,j_2}-\hat{\mu}_{i_2,j_2})\right|$$
where $\hat{\mu}_{i_1,j_1}$ (and similarly for others) is the measured performance per question of participant $i_1$ in context $j_1$ at the end of the experiment. 

A linear regression on the performance of the Baseline established that disparity ($\delta$) did not show a significant effect, $F(1, 34) = 3.65, p = 0.0645$ and accordingly disparity accounted for only $7.04\%$ of the explained variability. Whereas, a linear regression on the performance of Fair CB established that disparity significantly predicted its performance, $F(1, 34) = 32.9, p < 0.0001$ and the model explained $47.7\%$ of the variability in performance. The regression equation was: predicted performance $= 16.871 + 9.448 \delta$. Fig.~\ref{fig:disparity} shows the positive effect of disparity on performance in the Fair CB algorithm.
% , while it did not affect performance in the baseline. 

\subsubsection{Subjective Responses}

% \begin{figure*}[hbt!]
% \centering
% \includegraphics[width=0.3\textwidth]{fig/fair1_study.jpg}\label{fig:fairness1}
% \quad
% \includegraphics[width=0.3\textwidth]{fig/fair2_study.jpg}
% \label{fig:fairness2}
% \quad
% \includegraphics[width=0.3\textwidth]{fig/trust_study.jpg}
% \label{fig:trust}

% \caption{Responses to the subjective questions by each player (P1 are from USA, P2 from India) for the Baseline and Fair CB algorithms.}
% \label{fig:fairness}
% \vspace{-1 ex}
% \end{figure*}

\begin{figure}[t!]
\centering
\includegraphics[width=0.6\linewidth]{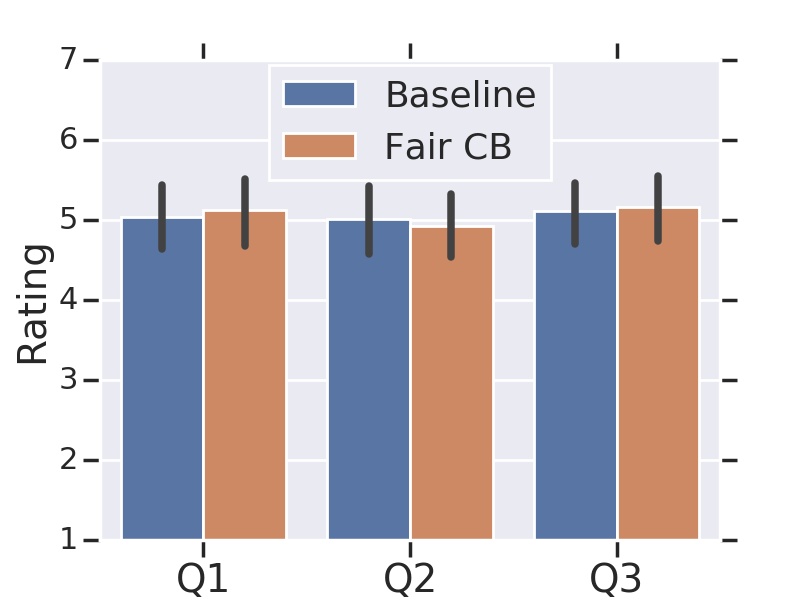}
\caption{Responses to the subjective questions in Table~\ref{tab:survey} by each player for the Baseline and Fair CB algorithms}
\label{fig:subjective}
\end{figure}

Out of the 37 pairs of participants that completed the quiz, 27 pairs (54 participants) answered all the subjective responses. We compare the responses of participants for the subjective questions given in Table~\ref{tab:survey} across the Baseline and Fair CB algorithms (Fig.~\ref{fig:subjective}).  

%A one-way ANOVA performed on the combined player responses across the two algorithms indicate no significant effect of the algorithm (Baseline or Fair CB) on the perceived fairness of the system (Q1: $p = 0.762$, Q2: $p = 0.76$). 
To test our hypothesis that the perceived fairness of the Fair CB algorithm is not worse than the Baseline,\footnote{We define ``not worse than'' using the concept of ``non-inferiority''~\cite{lesaffre2008superiority}.} a one-tailed paired t-test for a non-inferiority margin $\Delta = 0.5$ and a level of statistical significance $\alpha = 0.025$  showed that participants perceived the fairness of the Fair CB algorithm not worse than the Baseline for all questions ($p<0.0001$).

%In Fig.~\ref{fig:fairness}, we show the responses of all participants, grouped by their country of origin ($P1$ for USA, $P2$ for India). The figures show that participants rated their fairness and trust quite high in both algorithms. We attribute the slightly higher (albeit within the standard error) ratings for players $P2$ to potential differences in their cultural background. 

%The choice of algorithm also had no effect on the perceived trust in the system as seen in a one-way ANOVA ($p = 0.857$).
%However, the $P2$ players provided a slightly higher rating for trust in the system, than $P1$ players. The quiz was designed such that all participants from US were assigned as player 1 and all participants from India were assigned player 2. And the observed difference could be an artifact of the player background.

%\subsubsection{Subjective Responses}

We also asked participants to describe any difference they noticed in the way the questions were distributed between the two quizzes corresponding to the two algorithms. Users that did not have a clear disparity in their performance in the two contexts did not see a difference in the behaviour of the two algorithms. Users with greater disparity noticed a difference between the two algorithms, with some users even recognizing how each algorithm worked. Table~\ref{tab:comments} shows example responses for the users.

\section{Unknown Context Distributions}
\label{sec:unknown}
%!TEX root=main.tex
The Fair CB algorithm described in section~\ref{sec:algorithm} assumes that the context distribution $q$ is known to the learner. We provide an extension of our algorithm to the case where the context distribution $q$ is unknown. We include regret guarantee of the algorithm, while we leave empirical results for future work.

A natural idea is to maintain an empirical context distribution based on the observations and to use it as a proxy for $q$.
Specifically, to avoid changing the feasible set too often, we divide the entire horizon into $\order(\log_2 T)$ epochs, where epoch $k$ contains rounds $\tau_k, \ldots, \tau_{k+1}-1$ with $\tau_k = 2^{k-1}$.
Within epoch $k>1$, we let $q_k$ be the empirical context distribution using observations from the last $k-1$ epochs:
\begin{equation}\label{eqn:q_k}
q_k(j) = \frac{1}{\tau_k - 1}\sum_{t=1}^{\tau_k - 1}\one\{j_t = j\}, \;\forall j\in[M].
\end{equation}
Note that by standard concentration argument (specifically Bernstein inequalities and union bound), we have with probability at least $1-1/T$,
\begin{equation}\label{eqn:concentration}
\|q - q_k\|_1 \leq \epsilon_k \triangleq 4\sqrt{\frac{M\ln(TM)}{\tau_k - 1}} + \frac{2M\ln (TM)}{\tau_k - 1}.
\end{equation}
Accordingly, for epoch $k>1$ we define the feasible set $\Omega_k$ as
\begin{equation}\label{eqn:Omega_k}
\Omega_k = \begin{Bmatrix}
\left.\begin{matrix}
P = (p^1,...,p^M) 
\end{matrix}\left|\begin{matrix}
p^1,...,p^M \in \Delta_K \ \textrm{and}\\ 
\sum_{j=1}^M q_k(j)p^j(i) \geq v - \epsilon_k, \forall i \in [K]\end{matrix}\right.\right.
\end{Bmatrix},
\end{equation}
where we introduce a small slack $\epsilon_k$ to the fairness constraint $v$.
The reason of relaxing the constraint is to make sure that $\Omega_k$ always contains $\Omega$ with high probability.
Indeed, conditioning on the event Eq.~\eqref{eqn:concentration}, for any $P \in \Omega$ we have
\[
\sum_{j=1}^M q_k(j)p^j(i)
\geq \sum_{j=1}^M q(j)p^j(i) - \|q - q_k\|_1 \geq v - \epsilon_k
\]
and thus $P \in \Omega_k$.
On the other hand, relaxing the constraint means that the algorithm no longer always strictly satisfies the fairness requirement.
Instead, we measure the fairness of the algorithm by the average amount of violation of the fairness constraint, defined as
\[
\text{Vio} = \E\left[\frac{1}{T}\sum_{t=1}^T\max\left\{0, v - \min_{i\in[K]}\sum_{j=1}^M q(j)p^j_t(i)\right\}\right]
\]
where $p^j_t$ is again the distribution of arm $i_t$ given the history and $j_t = j$.

Our final algorithm simply runs a new instance of Algorithm~\ref{algorithm: known context} with feasible set $\Omega_k$ on epoch $k$.
See Algorithm~\ref{algorithm: unknown context} for the pseudocode.
In the following theorem, we show that the algorithm ensures the same regret bound while keeping the per-round fairness violation to be arbitrarily small as long as $T$ is large enough.

\begin{algorithm}[t]
 \caption{Fairness CB with Unknown Context Distribution}
\label{algorithm: unknown context}
\begin{algorithmic}[1]
\State \textbf{Input:} fairness constraint parameter $v$
\State \textbf{Define:} $\tau_k = 2^{k-1}$, $\Psi_k(P) =  \frac{1}{\eta_k}\sum_{j=1}^M\sum_{i=1}^K \psi(p^j(i))$ where $\psi(p) = p\ln p$ and $\eta_k = \sqrt{M\ln K/(\tau_k K)}$
\State For $t=1$, sample an arm uniformly at random
\For{$k=2,3,\ldots$}%\Comment{$k$ indexes a block with length $\tau_k$}
    \State Update $q_k$ and $\Omega_k$ according to Eq.~\eqref{eqn:q_k} and Eq.~\eqref{eqn:Omega_k}
    \For{$t=\tau_k,\ldots,\tau_{k+1}-1 $}
          \State Compute $P_{t} = \argmin_{P \in \Omega_k} \sum_{s=\tau_k}^{t-1} \inner{p^{j_s},\hat{l}_s} + \Psi_k(P)$
        \State Observe $j_t$ and play $i_t \sim p_t^{j_t}$
        \State Construct loss estimator $\hat{l}_t(i) = \frac{l_t(i)}{p_t^{j_t}(i)}\one\{i_t=i\}, \;\forall i$             
    \EndFor
\EndFor
\end{algorithmic}
\end{algorithm}

\begin{theorem}
Algorithm~\ref{algorithm: unknown context} ensures 
\[
\text{Reg} = \order\left(\sqrt{TMK\ln K}\right) \text{ and }\; \text{Vio} = \order\left(\sqrt{\frac{M\ln(TM)}{T}} + \frac{M\ln(TM)\ln T}{T}\right).
\]
\end{theorem}
\begin{proof}
Clearly we only need to condition on the event Eq.~\eqref{eqn:concentration} since it happens with probability at least $1-1/T$.
With the fact $P_* \in \Omega_k$ for all $k$, the regret guarantee is a simple application of Theorem~\ref{thm:regret}.
Indeed, let $K = \order(\log_2 T)$ be the total number of epochs, we have
\begin{align*}
\text{Reg} &= \sum_{k=1}^K \sum_{t=\tau_k}^{\min\{\tau_{k+1}-1,T\}} \E\left[\inner{p_t^{j_t}-p_*^{j_t},l_t} \right] \\
&= \sum_{k=1}^K \order\left(\sqrt{\tau_k MK\ln K}\right) = 
\order\left(\sqrt{TMK\ln K}\right).
\end{align*}
The amount of violation is also clear due to the construction of $\Omega_k$:
\begin{align*}
\text{Vio} \leq \frac{1}{T}\sum_{k=1}^K \tau_k \epsilon_k = \order\left(\sqrt{\frac{M\ln(TM)}{T}} + \frac{M\ln(TM)\ln T}{T}\right).
\end{align*}
This finishes the proof.
\end{proof}

\section{Discussion}\label{sec:discussion}
%!TEX root = main.tex

%-------------------------------------------------------------------------------
%\noindent\textbf{Design Recommendations.} 
%This work shows how different loss distributions, contexts and fairness affect  performance.
 We view our findings as valuable considerations regarding AI systems that make fair allocation decisions to multiple users.
Theoretically, we show how the classic FTRL framework can be naturally generalized to ensure fairness and we rigorously analyze the performance of our proposed algorithms in terms of both regret guarantee and fairness violation (in the case of unknown context distribution).

Empirically, our first finding is that increasing fairness results in worse performance, when there is one user who is outperformed in all contexts. On the other hand, if there exists a context where a user outperforms all others, whether fairness will affect performance depends on the distribution of contexts. If that context appears frequently enough for the desired fairness constraint to be satisfied, performance will not be affected. 

We also found that having a fair algorithm with no statistical assumptions about the process generating the losses is particularly beneficial in adversarial domains. Interestingly, increasing fairness in our adversarial setting was beneficial to the Fair UCB algorithm, since fairness reduced the reliance on the optimistic bounds that was exploited by the adversary. 

%We find this an interesting result that warrants further investigation. 

%the lack of Fair-CB algorithm leverages n adversarial domains that exploit the stochastic environment assumption. On the other hand, increasing fairness in that case may be beneficial, since it reduces reliance on the optimistic bounds that can be exploited by the adversary.  
%In that case, it is important to account for the trade-offs between having a fair system that guarantees a minimum rate of task allocation, with a system that performs optimally, ignoring the not top-performing users.

%This suggests that accounting for contexts \emph{mitigates the negative effect of fairness on performance}, since the proposed algorithm imposes the fairness constraint in the contexts that each user is best at. 

%may result in assigning more resources to overall worse performing team-members.

Finally, the benefit of the context-based algorithm depends on the disparity between users, that is how much they differ in their performance on each context. In our user study, Fair CB performed best for pairs of participants where each participant was better on one context and worse on another.

\noindent\textbf{Future Directions.} 
%
%In our theoretical contribution, we have assumed the same minimum allocation rate for all arms. It may be desirable to have different minimum rates for each user, and one can easily extend the proposed algorithms and theoretical results by redefining the feasible set $\Omega$ accordingly. 
%
We are excited to further investigate how our findings can generalize beyond online game settings, in domains where multiple users interact with a physically embodied robot~\cite{jung2018robot}: for instance, a robot receptionist greeting customers, an assistive robot in a stroke care facility helping patients eating a meal, or a factory robot delivering parts to workers. 

% Previous Wizard of Oz experiments on a physical human-robot collaboration setting have shown the effect of a robot's distribution of resources on team satisfaction~\cite{}, and we are excited to explore further  on deployed autonomous systems.

\noindent\textbf{Conclusion.} Overall, we are excited to have brought about a better understanding of the interplay between contexts, fairness and performance in task allocation settings. Designing AI systems that ensure and demonstrate fairness when interacting with people is critical to their acceptance, and deriving theoretical and experimental foundations for these systems is yet an under-served aspect in Human-AI Interaction.

%                                   % the textheight of the last page by a suitable amount.
%                                   % This command does not take effect until the next page
%                                   % so it should come on the page before the last. Make
%                                   % sure that you do not shorten the textheight too much.

% %%%%%%%%%%%%%%%%%%%%%%%%%%%%%%%%%%%%%%%%%%%%%%%%%%%%%%%%%%%%%%%%%%%%%%%%%%%%%%%%

% %%%%%%%%%%%%%%%%%%%%%%%%%%%%%%%%%%%%%%%%%%%%%%%%%%%%%%%%%%%%%%%%%%%%%%%%%%%%%%%%

%%%%%%%%%%%%%%%%%%%%%%%%%%%%%%%%%%%%%%%%%%%%%%%%%%%%%%%%%%%%%%%%%%%%%%%%%%%%%%%%

\bibliographystyle{ACM-Reference-Format}  % do not change this line!

\bibliography{references}


\begin{thebibliography}{00}

%%% the Web address.

\ifx \showCODEN    \undefined \def \showCODEN     #1{\unskip}     \fi
\ifx \showDOI      \undefined \def \showDOI       #1{#1}\fi
\ifx \showISBNx    \undefined \def \showISBNx     #1{\unskip}     \fi
\ifx \showISBNxiii \undefined \def \showISBNxiii  #1{\unskip}     \fi
\ifx \showISSN     \undefined \def \showISSN      #1{\unskip}     \fi
\ifx \showLCCN     \undefined \def \showLCCN      #1{\unskip}     \fi
\ifx \shownote     \undefined \def \shownote      #1{#1}          \fi
\ifx \showarticletitle \undefined \def \showarticletitle #1{#1}   \fi
\ifx \showURL      \undefined \def \showURL       {\relax}        \fi
% invisible to TeX
\providecommand\bibfield[2]{#2}
\providecommand\bibinfo[2]{#2}
\providecommand\natexlab[1]{#1}
\providecommand\showeprint[2][]{arXiv:#2}

\bibitem[\protect\citeauthoryear{Abbasi-Yadkori, P{\'a}l, and
  Szepesv{\'a}ri}{Abbasi-Yadkori et~al\mbox{.}}{2011}]%
        {abbasi2011improved}
\bibfield{author}{\bibinfo{person}{Yasin Abbasi-Yadkori},
  \bibinfo{person}{D{\'a}vid P{\'a}l}, {and} \bibinfo{person}{Csaba
  Szepesv{\'a}ri}.} \bibinfo{year}{2011}\natexlab{}.
\newblock \showarticletitle{Improved algorithms for linear stochastic bandits}.
  In \bibinfo{booktitle}{{\em Advances in Neural Information Processing
  Systems}}. \bibinfo{pages}{2312--2320}.
\newblock


\bibitem[\protect\citeauthoryear{Abernethy, Lee, and Tewari}{Abernethy
  et~al\mbox{.}}{2015}]%
        {abernethy2015fighting}
\bibfield{author}{\bibinfo{person}{Jacob~D Abernethy}, \bibinfo{person}{Chansoo
  Lee}, {and} \bibinfo{person}{Ambuj Tewari}.} \bibinfo{year}{2015}\natexlab{}.
\newblock \showarticletitle{Fighting bandits with a new kind of smoothness}. In
  \bibinfo{booktitle}{{\em Advances in Neural Information Processing Systems}}.
  \bibinfo{pages}{2197--2205}.
\newblock


\bibitem[\protect\citeauthoryear{Agarwal, Hsu, Kale, Langford, Li, and
  Schapire}{Agarwal et~al\mbox{.}}{2014}]%
        {agarwal2014taming}
\bibfield{author}{\bibinfo{person}{Alekh Agarwal}, \bibinfo{person}{Daniel
  Hsu}, \bibinfo{person}{Satyen Kale}, \bibinfo{person}{John Langford},
  \bibinfo{person}{Lihong Li}, {and} \bibinfo{person}{Robert Schapire}.}
  \bibinfo{year}{2014}\natexlab{}.
\newblock \showarticletitle{Taming the monster: A fast and simple algorithm for
  contextual bandits}. In \bibinfo{booktitle}{{\em International Conference on
  Machine Learning}}. \bibinfo{pages}{1638--1646}.
\newblock


\bibitem[\protect\citeauthoryear{Auer, Cesa-Bianchi, Freund, and Schapire}{Auer
  et~al\mbox{.}}{2002}]%
        {auer2002nonstochastic}
\bibfield{author}{\bibinfo{person}{Peter Auer}, \bibinfo{person}{Nicolo
  Cesa-Bianchi}, \bibinfo{person}{Yoav Freund}, {and} \bibinfo{person}{Robert~E
  Schapire}.} \bibinfo{year}{2002}\natexlab{}.
\newblock \showarticletitle{The nonstochastic multiarmed bandit problem}.
\newblock \bibinfo{journal}{{\em SIAM journal on computing\/}}
  \bibinfo{volume}{32}, \bibinfo{number}{1} (\bibinfo{year}{2002}),
  \bibinfo{pages}{48--77}.
\newblock


\bibitem[\protect\citeauthoryear{Bubeck, Cesa-Bianchi, et~al\mbox{.}}{Bubeck
  et~al\mbox{.}}{2012}]%
        {bubeck2012regret}
\bibfield{author}{\bibinfo{person}{S{\'e}bastien Bubeck},
  \bibinfo{person}{Nicolo Cesa-Bianchi}, {et~al\mbox{.}}}
  \bibinfo{year}{2012}\natexlab{}.
\newblock \showarticletitle{Regret analysis of stochastic and nonstochastic
  multi-armed bandit problems}.
\newblock \bibinfo{journal}{{\em Foundations and Trends{\textregistered} in
  Machine Learning\/}} \bibinfo{volume}{5}, \bibinfo{number}{1}
  (\bibinfo{year}{2012}), \bibinfo{pages}{1--122}.
\newblock


\bibitem[\protect\citeauthoryear{Bubeck, Stoltz, and Yu}{Bubeck
  et~al\mbox{.}}{2011}]%
        {bubeck2011lipschitz}
\bibfield{author}{\bibinfo{person}{S{\'e}bastien Bubeck},
  \bibinfo{person}{Gilles Stoltz}, {and} \bibinfo{person}{Jia~Yuan Yu}.}
  \bibinfo{year}{2011}\natexlab{}.
\newblock \showarticletitle{Lipschitz bandits without the Lipschitz constant}.
  In \bibinfo{booktitle}{{\em International Conference on Algorithmic Learning
  Theory}}. Springer, \bibinfo{pages}{144--158}.
\newblock


\bibitem[\protect\citeauthoryear{Cesa-Bianchi, Gaillard, Gentile, and
  Gerchinovitz}{Cesa-Bianchi et~al\mbox{.}}{2017}]%
        {cesa2017algorithmic}
\bibfield{author}{\bibinfo{person}{Nicol{\`o} Cesa-Bianchi},
  \bibinfo{person}{Pierre Gaillard}, \bibinfo{person}{Claudio Gentile}, {and}
  \bibinfo{person}{S{\'e}bastien Gerchinovitz}.}
  \bibinfo{year}{2017}\natexlab{}.
\newblock \showarticletitle{Algorithmic chaining and the role of partial
  feedback in online nonparametric learning}.
\newblock \bibinfo{journal}{{\em arXiv preprint arXiv:1702.08211\/}}
  (\bibinfo{year}{2017}).
\newblock


\bibitem[\protect\citeauthoryear{Chu, Li, Reyzin, and Schapire}{Chu
  et~al\mbox{.}}{2011}]%
        {chu2011contextual}
\bibfield{author}{\bibinfo{person}{Wei Chu}, \bibinfo{person}{Lihong Li},
  \bibinfo{person}{Lev Reyzin}, {and} \bibinfo{person}{Robert Schapire}.}
  \bibinfo{year}{2011}\natexlab{}.
\newblock \showarticletitle{Contextual bandits with linear payoff functions}.
  In \bibinfo{booktitle}{{\em Proceedings of the Fourteenth International
  Conference on Artificial Intelligence and Statistics}}.
  \bibinfo{pages}{208--214}.
\newblock


\bibitem[\protect\citeauthoryear{Claure, Chen, Modi, Jung, and
  Nikolaidis}{Claure et~al\mbox{.}}{2019}]%
        {claure2019reinforcement}
\bibfield{author}{\bibinfo{person}{Houston Claure}, \bibinfo{person}{Yifang
  Chen}, \bibinfo{person}{Jignesh Modi}, \bibinfo{person}{Malte Jung}, {and}
  \bibinfo{person}{Stefanos Nikolaidis}.} \bibinfo{year}{2019}\natexlab{}.
\newblock \showarticletitle{Reinforcement Learning with Fairness Constraints
  for Resource Distribution in Human-Robot Teams}.
\newblock \bibinfo{journal}{{\em arXiv preprint arXiv:1907.00313\/}}
  (\bibinfo{year}{2019}).
\newblock


\bibitem[\protect\citeauthoryear{Dwork, Hardt, Pitassi, Reingold, and
  Zemel}{Dwork et~al\mbox{.}}{2012}]%
        {dwork2012fairness}
\bibfield{author}{\bibinfo{person}{Cynthia Dwork}, \bibinfo{person}{Moritz
  Hardt}, \bibinfo{person}{Toniann Pitassi}, \bibinfo{person}{Omer Reingold},
  {and} \bibinfo{person}{Richard Zemel}.} \bibinfo{year}{2012}\natexlab{}.
\newblock \showarticletitle{Fairness through awareness}. In
  \bibinfo{booktitle}{{\em Proceedings of the 3rd innovations in theoretical
  computer science conference}}. ACM, \bibinfo{pages}{214--226}.
\newblock


\bibitem[\protect\citeauthoryear{Joseph, Kearns, Morgenstern, Neel, and
  Roth}{Joseph et~al\mbox{.}}{2016b}]%
        {joseph2016fair}
\bibfield{author}{\bibinfo{person}{Matthew Joseph}, \bibinfo{person}{Michael
  Kearns}, \bibinfo{person}{Jamie Morgenstern}, \bibinfo{person}{Seth Neel},
  {and} \bibinfo{person}{Aaron Roth}.} \bibinfo{year}{2016}\natexlab{b}.
\newblock \showarticletitle{Fair algorithms for infinite and contextual
  bandits}.
\newblock \bibinfo{journal}{{\em arXiv preprint arXiv:1610.09559\/}}
  (\bibinfo{year}{2016}).
\newblock


\bibitem[\protect\citeauthoryear{Joseph, Kearns, Morgenstern, and Roth}{Joseph
  et~al\mbox{.}}{2016a}]%
        {joseph2016fairness}
\bibfield{author}{\bibinfo{person}{Matthew Joseph}, \bibinfo{person}{Michael
  Kearns}, \bibinfo{person}{Jamie~H Morgenstern}, {and} \bibinfo{person}{Aaron
  Roth}.} \bibinfo{year}{2016}\natexlab{a}.
\newblock \showarticletitle{Fairness in learning: Classic and contextual
  bandits}. In \bibinfo{booktitle}{{\em Advances in Neural Information
  Processing Systems}}. \bibinfo{pages}{325--333}.
\newblock


\bibitem[\protect\citeauthoryear{Jung, DiFranzo, Stoll, Shen, Lawrence, and
  Claure}{Jung et~al\mbox{.}}{2018}]%
        {jung2018robot}
\bibfield{author}{\bibinfo{person}{Malte~F Jung}, \bibinfo{person}{Dominic
  DiFranzo}, \bibinfo{person}{Brett Stoll}, \bibinfo{person}{Solace Shen},
  \bibinfo{person}{Austin Lawrence}, {and} \bibinfo{person}{Houston Claure}.}
  \bibinfo{year}{2018}\natexlab{}.
\newblock \showarticletitle{Robot Assisted Tower Construction-A Resource
  Distribution Task to Study Human-Robot Collaboration and Interaction with
  Groups of People}.
\newblock \bibinfo{journal}{{\em arXiv preprint arXiv:1812.09548\/}}
  (\bibinfo{year}{2018}).
\newblock


\bibitem[\protect\citeauthoryear{Kleinberg, Slivkins, and Upfal}{Kleinberg
  et~al\mbox{.}}{2008}]%
        {kleinberg2008multi}
\bibfield{author}{\bibinfo{person}{Robert Kleinberg},
  \bibinfo{person}{Aleksandrs Slivkins}, {and} \bibinfo{person}{Eli Upfal}.}
  \bibinfo{year}{2008}\natexlab{}.
\newblock \showarticletitle{Multi-armed bandits in metric spaces}. In
  \bibinfo{booktitle}{{\em Proceedings of the fortieth annual ACM symposium on
  Theory of computing}}. ACM, \bibinfo{pages}{681--690}.
\newblock


\bibitem[\protect\citeauthoryear{Kutner, Nachtsheim, Neter, Li,
  et~al\mbox{.}}{Kutner et~al\mbox{.}}{2005}]%
        {kutner2005applied}
\bibfield{author}{\bibinfo{person}{Michael~H Kutner},
  \bibinfo{person}{Christopher~J Nachtsheim}, \bibinfo{person}{John Neter},
  \bibinfo{person}{William Li}, {et~al\mbox{.}}}
  \bibinfo{year}{2005}\natexlab{}.
\newblock \bibinfo{booktitle}{{\em Applied linear statistical models}}.
  Vol.~\bibinfo{volume}{5}.
\newblock \bibinfo{publisher}{McGraw-Hill Irwin Boston}.
\newblock


\bibitem[\protect\citeauthoryear{Langford and Zhang}{Langford and
  Zhang}{2007}]%
        {langford2007epoch}
\bibfield{author}{\bibinfo{person}{John Langford} {and} \bibinfo{person}{Tong
  Zhang}.} \bibinfo{year}{2007}\natexlab{}.
\newblock \showarticletitle{The epoch-greedy algorithm for contextual
  multi-armed bandits}. In \bibinfo{booktitle}{{\em Proceedings of the 20th
  International Conference on Neural Information Processing Systems}}.
  Citeseer, \bibinfo{pages}{817--824}.
\newblock


\bibitem[\protect\citeauthoryear{Lesaffre}{Lesaffre}{2008}]%
        {lesaffre2008superiority}
\bibfield{author}{\bibinfo{person}{Emmanuel Lesaffre}.}
  \bibinfo{year}{2008}\natexlab{}.
\newblock \showarticletitle{Superiority, equivalence, and non-inferiority
  trials.}
\newblock \bibinfo{journal}{{\em Bulletin of the NYU hospital for joint
  diseases\/}} \bibinfo{volume}{66}, \bibinfo{number}{2}
  (\bibinfo{year}{2008}).
\newblock


\bibitem[\protect\citeauthoryear{Li, Liu, and Ji}{Li et~al\mbox{.}}{2019}]%
        {li2019combinatorial}
\bibfield{author}{\bibinfo{person}{Fengjiao Li}, \bibinfo{person}{Jia Liu},
  {and} \bibinfo{person}{Bo Ji}.} \bibinfo{year}{2019}\natexlab{}.
\newblock \showarticletitle{Combinatorial sleeping bandits with fairness
  constraints}. In \bibinfo{booktitle}{{\em IEEE INFOCOM 2019-IEEE Conference
  on Computer Communications}}. IEEE, \bibinfo{pages}{1702--1710}.
\newblock


\bibitem[\protect\citeauthoryear{Li, Chu, Langford, and Schapire}{Li
  et~al\mbox{.}}{2010}]%
        {li2010contextual}
\bibfield{author}{\bibinfo{person}{Lihong Li}, \bibinfo{person}{Wei Chu},
  \bibinfo{person}{John Langford}, {and} \bibinfo{person}{Robert~E Schapire}.}
  \bibinfo{year}{2010}\natexlab{}.
\newblock \showarticletitle{A contextual-bandit approach to personalized news
  article recommendation}. In \bibinfo{booktitle}{{\em Proceedings of the 19th
  international conference on World wide web}}. ACM, \bibinfo{pages}{661--670}.
\newblock


\bibitem[\protect\citeauthoryear{Liu, Radanovic, Dimitrakakis, Mandal, and
  Parkes}{Liu et~al\mbox{.}}{2017}]%
        {liu2017calibrated}
\bibfield{author}{\bibinfo{person}{Yang Liu}, \bibinfo{person}{Goran
  Radanovic}, \bibinfo{person}{Christos Dimitrakakis},
  \bibinfo{person}{Debmalya Mandal}, {and} \bibinfo{person}{David~C Parkes}.}
  \bibinfo{year}{2017}\natexlab{}.
\newblock \showarticletitle{Calibrated fairness in bandits}.
\newblock \bibinfo{journal}{{\em arXiv preprint arXiv:1707.01875\/}}
  (\bibinfo{year}{2017}).
\newblock


\bibitem[\protect\citeauthoryear{Patil, Ghalme, Nair, and Narahari}{Patil
  et~al\mbox{.}}{2019}]%
        {patil2019achieving}
\bibfield{author}{\bibinfo{person}{Vishakha Patil}, \bibinfo{person}{Ganesh
  Ghalme}, \bibinfo{person}{Vineet Nair}, {and} \bibinfo{person}{Y Narahari}.}
  \bibinfo{year}{2019}\natexlab{}.
\newblock \showarticletitle{Achieving Fairness in the Stochastic Multi-armed
  Bandit Problem}.
\newblock \bibinfo{journal}{{\em arXiv preprint arXiv:1907.10516\/}}
  (\bibinfo{year}{2019}).
\newblock


\bibitem[\protect\citeauthoryear{Slivkins}{Slivkins}{2014}]%
        {slivkins2014contextual}
\bibfield{author}{\bibinfo{person}{Aleksandrs Slivkins}.}
  \bibinfo{year}{2014}\natexlab{}.
\newblock \showarticletitle{Contextual bandits with similarity information}.
\newblock \bibinfo{journal}{{\em The Journal of Machine Learning Research\/}}
  \bibinfo{volume}{15}, \bibinfo{number}{1} (\bibinfo{year}{2014}),
  \bibinfo{pages}{2533--2568}.
\newblock


\bibitem[\protect\citeauthoryear{Syrgkanis, Luo, Krishnamurthy, and
  Schapire}{Syrgkanis et~al\mbox{.}}{2016}]%
        {SyrgkanisLuKrSc16}
\bibfield{author}{\bibinfo{person}{Vasilis Syrgkanis}, \bibinfo{person}{Haipeng
  Luo}, \bibinfo{person}{Akshay Krishnamurthy}, {and} \bibinfo{person}{Robert~E
  Schapire}.} \bibinfo{year}{2016}\natexlab{}.
\newblock \showarticletitle{Improved Regret Bounds for Oracle-Based Adversarial
  Contextual Bandits}. In \bibinfo{booktitle}{{\em Advances in Neural
  Information Processing Systems}}.
\newblock


\bibitem[\protect\citeauthoryear{Wu, Iyer, and Wang}{Wu et~al\mbox{.}}{2018}]%
        {wu2018learning}
\bibfield{author}{\bibinfo{person}{Qingyun Wu}, \bibinfo{person}{Naveen Iyer},
  {and} \bibinfo{person}{Hongning Wang}.} \bibinfo{year}{2018}\natexlab{}.
\newblock \showarticletitle{Learning contextual bandits in a non-stationary
  environment}. In \bibinfo{booktitle}{{\em The 41st International ACM SIGIR
  Conference on Research \& Development in Information Retrieval}}. ACM,
  \bibinfo{pages}{495--504}.
\newblock


\end{thebibliography}

% %%%%%%%%%%%%%%%%%%%%%%%%%%%%%%%%%%%%%%%%%%%%%%%%%%%%%%%%%%%%%%%%%%%%%%%%%%%%%%%%
%\section*{APPENDIX}
%\input{appendix.tex}

\end{document}